\documentclass[10pt,journal,compsoc]{IEEEtran}

\usepackage{times}
\usepackage{epsfig}
\usepackage{graphicx}
\usepackage{amsmath}
\usepackage{amssymb}

\usepackage{booktabs}
\usepackage{amsthm}
\theoremstyle{plain}
\usepackage[compress]{cite}
\usepackage{graphics}
\usepackage{epstopdf}
\usepackage{epsfig}
\usepackage{lineno,hyperref}
\usepackage{array}
\usepackage{amsmath,amssymb}
\usepackage{mathrsfs}
\usepackage{subfigure} 
\usepackage{multirow}
\usepackage{xcolor}
\usepackage{bm}
\usepackage{eqparbox}
\usepackage{algorithm}
\usepackage{algorithmic}
\usepackage{amsfonts}
\usepackage{enumerate}
\usepackage{indentfirst}
\usepackage{url}
\newtheorem{myTheorem}{Theorem}

\newtheorem{myLemma}{Lemma}
\newtheorem{myAssumption}{Assumption}

\usepackage{bbm}
\allowdisplaybreaks[4]

\usepackage{ragged2e}
\begin{document}

\title{Adaptive Graph Auto-Encoder for General Data Clustering}

\author{Xuelong Li, \IEEEmembership{~Fellow,~IEEE}, Hongyuan Zhang, and Rui Zhang,\IEEEmembership{~Member,~IEEE}



\thanks{The authors are with the School of Artificial Intelligence, Optics and Electronics (iOPEN), Northwestern Polytechnical University, Xi'an 710072, P.R. China. They are also with the Key Laboratory of Intelligent Interaction and Applications (Northwestern Polytechnical University), Ministry of Industry and Information Technology, Xi'an 710072, P. R. China.}




\thanks{This work is supported by The National Natural Science Foundation of China (No. 61871470). }

\thanks{Corresponding author: Xuelong Li.}

\thanks{The codes can be downloaded from \url{https://github.com/hyzhang98/AdaGAE}.}

\thanks{E-mail: li@nwpu.edu.cn; hyzhang98@gmail.com.}

\thanks{
    \copyright 2021 IEEE.  Personal use of this material is permitted.  Permission from IEEE must be obtained for all other uses, in any current or future media, including reprinting/republishing this material for advertising or promotional purposes, creating new collective works, for resale or redistribution to servers or lists, or reuse of any copyrighted component of this work in other works.
}

}

\IEEEtitleabstractindextext{
\justifying  
\begin{abstract}
    Graph-based clustering plays an important role in the clustering area. 
    Recent studies about graph neural networks (\textit{GNN}) have achieved impressive success on graph-type data. 
    However, in general clustering tasks, the graph structure of data does not exist 
    such that GNN can not be applied to clustering directly and the strategy to construct a graph is crucial for performance. 
    Therefore, how to extend GNN into general clustering tasks is an attractive problem.  
    In this paper, we propose a graph auto-encoder for general data clustering, \textit{AdaGAE}, which constructs the graph adaptively according to the generative perspective of graphs. 
    The adaptive process is designed to induce the model to exploit the high-level information behind data and utilize the non-Euclidean structure sufficiently.
    Importantly, we find that the simple update of the graph will result in severe degeneration, 
    which can be concluded as \textit{better reconstruction means worse update}.
    We provide rigorous analysis theoretically and empirically. Then we further design a novel mechanism to avoid the collapse. 
    Via extending the generative graph models to general type data, a graph auto-encoder with a novel decoder is devised and the weighted graphs can be also applied to GNN. 
    AdaGAE performs well and stably in different scale and type datasets. 
    Besides, it is insensitive to the initialization of parameters and requires no pretraining.
\end{abstract}


}

\maketitle

\section{Introduction}
Clustering, which intends to group data points without supervision information, is one of the most fundamental tasks in machine learning. 
As well as the well-known $k$-means \cite{clustering_base, FCM, MJP}, graph-based clustering \cite{SC, RatioCut, CAN} is also a representative kind of clustering method. 
Graph-based clustering methods can capture manifold information so that they are available for the non-Euclidean type data, which is not provided by $k$-means. Therefore, they are widely used in practice. Due to the success of deep learning, how to combine neural networks and traditional clustering models has been studied a lot \cite{SpectralNet, DEC, DFKM}. In particular, CNN-based clustering models have been extensively investigated \cite{JULE,DEPICT,DualDeepClustering}. However, the convolution operation may be unavailable on other kinds of datasets, \textit{e.g.}, text, social network, signal, data mining, \textit{etc}.

Network embedding is a fundamental task for graph type data such as recommendation systems, social networks, \textit{etc}. 
The goal is to map nodes of a given graph into latent features (namely embedding) such that the learned embedding can be utilized on node classification, node clustering, and link prediction. 
Roughly speaking, the network embedding approaches can be classified into 2 categories: generative models \cite{GraphGAN, DeepWalk} and discriminative models \cite{DNGR,SDNE}. The former tries to model a connectivity distribution for each node while the latter learns to distinguish whether an edge exists between two nodes directly. 
%
In recent years, graph neural networks (\textit{GNN}) \cite{GNN}, especially graph convolution neural networks (\textit{GCN}) \cite{GCN,ChebNet}, have attracted a mass of attention due to the success made in the neural networks area. GNNs extend classical neural networks into irregular data so that the deep information hidden in graphs is exploited sufficiently. In this paper, we only focus on GCNs and its variants. 
GCNs have shown superiority compared with traditional network embedding models. Similarly, graph auto-encoder (\textit{GAE}) \cite{GAE,MGAE,AttentionGAE} is developed to extend GCN into unsupervised learning.

\begin{figure*}[t]
    \centering
    \setlength{\abovecaptionskip}{2mm}
    \setlength{\belowcaptionskip}{0mm}
    \includegraphics[width=0.75\linewidth]{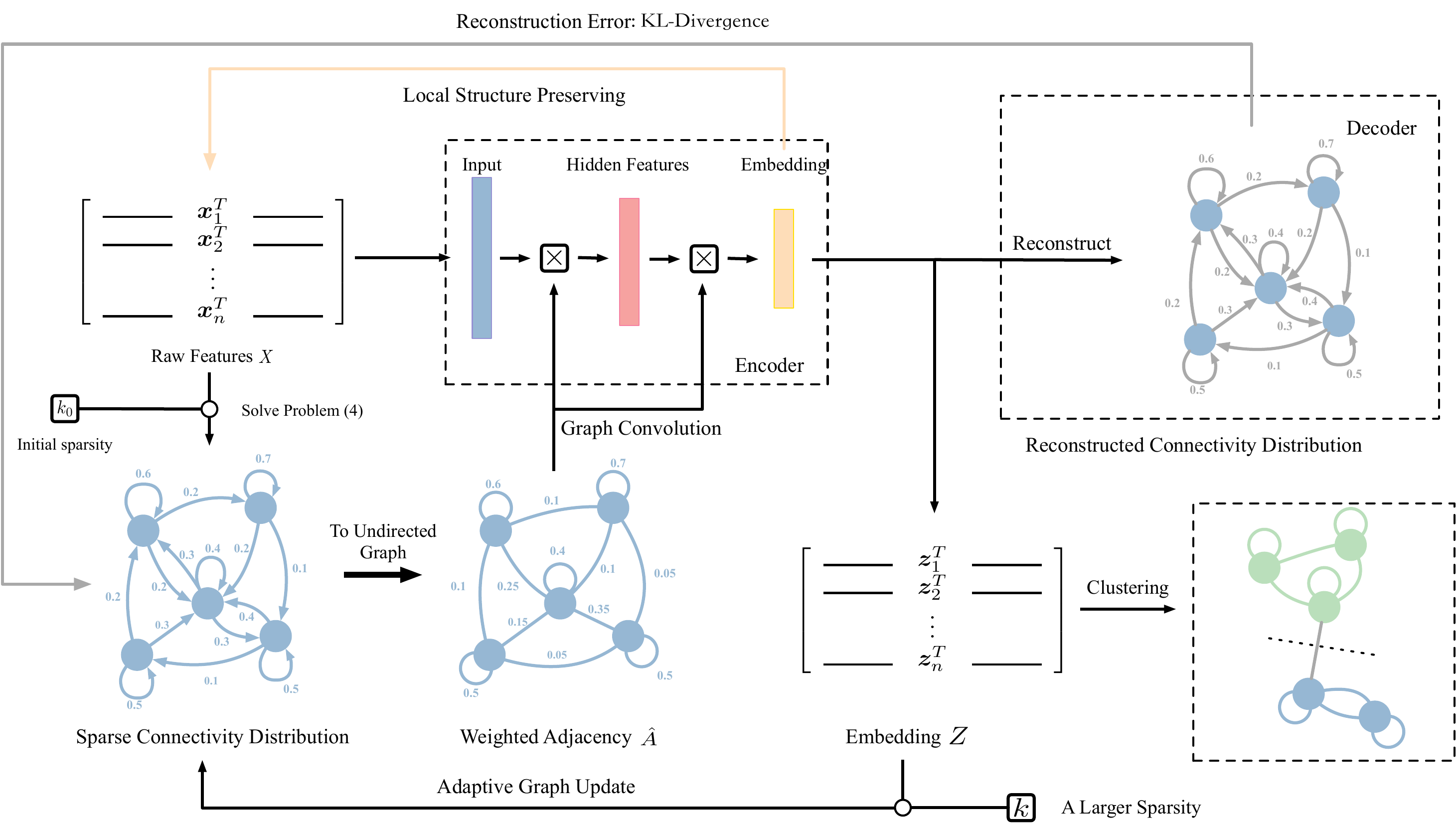}
    \caption{Framework of AdaGAE. $k_0$ is the initial sparsity. First, we construct a sparse graph via the generative model defined in Eq. (\ref{obj_CAN}). The learned graph is employed to apply the GAE designed for the weighted graphs. After training the GAE, we update the graph from the learned embedding with a larger sparsity, $k$. With the new graph, we re-train the GAE. These steps are repeated until the convergence. }
    \label{figure_framework}
\end{figure*}

However, the existing methods are limited to graph type data while no graph is provided for general data clustering. Since a large proportion of clustering methods are based on the graph, it is reasonable to consider how to employ GCN to promote the performance of graph-based clustering methods.
In this paper, we propose an Adaptive Graph Auto-Encoder (\textit{AdaGAE}) to extend graph auto-encoder into common scenarios. The main contributions are listed as follows:
\textbf{(1)} Via extending the generative graph models into general type data, GAE is naturally employed as the basic representation learning model and weighted graphs can be further applied to GAE as well. The connectivity distributions given by the generative perspective also inspires us to devise a novel architecture for decoders.
\textbf{(2)} As we utilize GAE to exploit the high-level information to construct a desirable graph, we find that the model suffers from a severe collapse due to the simple update of the graph. We analyze the degeneration theoretically and experimentally to understand the phenomenon. We further propose a simple but effective strategy to avoid it.
\textbf{(3)} AdaGAE is a scalable clustering model that works stably on different scale and type datasets, while the other deep clustering models usually fail when the training set is not large enough. Besides, it is insensitive to different initialization of parameters and needs no pretraining.

The paper is organized as: 
Section \ref{section_background} introduces 
the mathematical notations and revisits the works of deep clustering 
and graph auto-encoders.
Section \ref{subsection_perspective} introduces the generative graph models 
which is utilized to build GAE models for weighted graphs and
\ref{subsection_architecture} show the basic architecture of AdaGAE.
Section \ref{subsection_adaptive} shows a collapse phenomenon due to the naive update 
and then gives a strategy to solve it. 
All proofs are shown in Section \ref{section_proof} 
and details of experiments are provided in Section \ref{section_experiments}.

\section{Preliminary and Related Work} \label{section_background}
\vspace{-0.5mm}

\subsection{Notations}
In this paper, matrices and vectors are represented by uppercase and lowercase letters respectively. 
A graph is represented as $\mathcal G = (\mathcal V, \mathcal E, \mathcal W)$ and $|\cdot|$ is the size of some set. Vectors whose all elements equal 1 are represented as \textbf{1}.  If $\langle v_i, v_j\rangle \in \mathcal E$, then $\mathcal W_{ij} > 0$; otherwise, $\mathcal W_{ij} = 0$. For every node $v_i \in \mathcal V$, it is represented by a $d$-dimension vector $\bm x_i$ and thus, $\mathcal V$ can be also denoted by $X = [\bm x_1, \bm x_2, \cdots, \bm x_n]^T \in \mathbb R^{n \times d}$. The number of data points and clusters are represented as $n$ and $c$ respectively. 
All proofs are shown in Section \ref{section_proof}.
\vspace{-0.5mm}

\subsection{Deep Clustering}
An important topic in clustering filed is deep clustering, 
which utilizes neural networks to enhance the capacity of model. 
A fundamental model is auto-encoder (\textit{AE}) \cite{AE}, 
which has been widely used in diverse clustering methods \cite{DEC, DFKM}. 
Besides the AE-based models, SpectralNet \cite{SpectralNet} attempts to transform the core idea of spectral clustering to neural networks. 
Although SpectralNet has extended the spectral clustering into the deep version, 
it pays more attention to the simulation of spectral decomposition via neural networks 
but fails to investigate how to construct a high-quality graph via deep representations. 
In image clustering fields, CNN-based models have achieved impressive performance. 
For instance, JULE \cite{JULE} employs both CNN and RNN to obtain better representation. 
DEPICT \cite{DEPICT} and deep spectral clustering with dual networks \cite{DualDeepClustering} are based on convolution auto-encoders. 
However, they can only be applied on image datasets. 
Since we focus on clustering on both regular and irregular data, 
\textbf{these CNN-based models will not be regarded as main competitors}.

\subsection{Graph Auto-Encoder}\label{section_GAE}
In recent years, GCNs have been studied a lot to extend neural networks to graph type data. How to design a graph convolution operator is a key issue and has attracted a mass of attention. Most of them can be classified into 2 categories, spectral methods \cite{patchysan} and spatial methods\cite{spectralgcn}. 
In this paper, we focus on a simple but widely used convolution operator \cite{GCN}, which can be regarded as both the spectral operator and spatial operator. Formally, if the input of a graph convolution layer is $X \in \mathbb R^{n \times d}$ and the adjacency matrix is $A$, then the output is defined as 
\begin{equation}
    H = \varphi(\hat A X W) ,
\end{equation}
where $\varphi(\cdot)$ is certain activation function, $\hat A = \widetilde D^{-\frac{1}{2}} \widetilde A \widetilde D^{-\frac{1}{2}}$, $\widetilde A = A + I$, $\widetilde D$ denotes the degree matrix ($\widetilde D_{ii} = \sum_{j=1}^n \widetilde A_{ij}$), and $W$ denotes the parameters of GCN. It should be pointed out that $\widetilde A$ is a graph with self-loop for each node and $\hat A$ is the normalized adjacency matrix. More importantly, $\hat A X$ is equivalent to compute weighted means for each node with its first-order neighbors from the spatial aspect. To improve the performance, MixHop \cite{MixHopGCN} aims to mix information from different order neighbors and SGC \cite{SimplifyingGCN} tries to utilize higher-order neighbors. 
The capacity of GCN is also proved to some extent \cite{powerfulGCN}. 
GCN and its variants are usually used on semi-supervised learning. 

To apply graph convolution on unsupervised learning, GAE is proposed \cite{GAE}. 
GAE firstly transforms each node into latent representation (\textit{i.e.}, embedding) via GCN, and then aims to reconstruct some part of the input. GAEs proposed in \cite{GAE, AGAE, AttentionGAE} intend to reconstruct the adjacency via decoder while GAEs developed in \cite{MGAE} attempt to reconstruct the content. The difference is which extra mechanism (such as attention, adversarial learning, graph sharpness, \textit{etc.}) is used. 
EGAE \cite{EGAE} aims to design rational decoder and clustering module from theoretical aspect.   
SDCN \cite{SDCN} tries to apply GNN on general data clustering, 
but it simply constructs a graph manually and then fixes it during training.


\section{Proposed Model} \label{section_AdaGAE}
In this section, 
we elaborate the proposed model, Adaptive Graph Auto-Encoder (\textit{AdaGAE}) for general data clustering. 
The core idea of AdaGAE is illustrated in Figure \ref{figure_framework}. 
\subsection{Probabilistic Perspective of Weighted Graphs} \label{subsection_perspective}
In this paper, the underlying connectivity distribution of node $v_i$ is denoted by conditional probability $p(v | v_i)$ such that $\sum _{j=1}^n p(v_j | v_i) = 1$. From this perspective, a link can be regarded as a sampling result according to $p(v | v_i)$, which is the core assumption of the generative network embedding.

In general clustering scenarios, links between two nodes frequently do not exist. Therefore, we need to construct a weighted graph via some scheme. Since $p(v_j | v_i) \geq 0$, the probability can be regarded as valid weights. Note that $p(v_j | v_i) \neq p(v_i | v_j)$ usually holds, and therefore, the constructed graph should be viewed as a directed graph. Therefore, the construction of the weighted graph is equivalent to finding the underlying connectivity distribution.
The following assumption helps us to find an approximate connectivity distribution,
\begin{myAssumption} \label{assumption_dist}
    In an ideal situation, the representation of $v_i$ is analogous 
    to the one of $v_j$ if $p(v_i | v_j)$ is large.
\end{myAssumption}
Suppose that divergence between $v_i$ and $v_j$ is denoted by 
\begin{equation} \label{eq_divergence}
    d(v_i, v_j) = \|f(\bm x_i) - f(\bm x_j)\|_2^2,
\end{equation}
where $\bm x_i$ is the raw feature that describes the node $v_i$
and $f(\cdot)$ is a mapping that aims to find the optimal representation to 
cater for Assumption \ref{assumption_dist} .
Accordingly, we expect that 
\begin{equation}
    \min \limits_{p(\cdot | v_i), f(\cdot)} \sum \limits_{i=1}^n \mathbb{E}_{v_j \sim p(\cdot | v_i)} d(v_i, v_j) = \sum \limits_{i = 1}^n \sum \limits_{j=1}^n p(v_j | v_i) \cdot d_{ij} ,
\end{equation}
where $d_{ij} = d(v_i, v_j)$ for simplicity. 
To solve the above problem, the \textit{alternative method} is utilized.
However, it is impracticable to solve the above problem directly, 
as its subproblem regarding $p(\cdot | v_i)$ has a trivial solution: 
$p(v_{i} | v_i) = 1$ and $p(v_{j} | v_i) = 0$ if $j \neq i$. 
A universal method is to employ \textit{Regularization Loss Minimization}, and the objective is given as 
\begin{equation}
    \label{obj_RLM}
    \min \limits_{p(\cdot | v_i), f(\cdot)} \sum \limits_{i=1}^n \mathbb{E}_{v_j \sim p(\cdot | v_i)} d_{ij} + \mathcal R_0 (p(\cdot | v_i)) ,
\end{equation}
where $\mathcal R_0(\cdot)$ is some regularization term. 
Let $p_* (\cdot | v_i)$ be the prior distribution and 
$\mathcal{R}(\cdot)$ can be further formulated as 
\begin{equation}
    \mathcal{R}_0 (p(\cdot | v_i)) = \gamma \mathcal{R}(p(\cdot|v_i), p_* (\cdot | v_i)) .
\end{equation}
Without extra information, the prior distribution can be simply set as the 
uniform distribution, which is denoted by $\pi (v_j | v_i)$. 
Then, $\mathcal{R}(\cdot, \cdot)$ represents a function that measures two discrete distributions.
A common choice is Kullback-Leibler divergence (\textit{KL}-divergence).
However, we find that a more simple measurement, 
$\mathcal{R}(p(\cdot | v_i), \pi(\cdot | v_i))$, 
will provide more promising property, as 
\begin{equation}
    \mathcal{R}(p(\cdot | v_i), \pi (\cdot | v_i)) = \sum _{j=1}^n [p(v_j | v_i) - \pi (v_j | v_i)]^2 = \|\bm p_i\|_2^2 + C,
\end{equation}
where $\bm p_i = [p(v_1 | v_i), p(v_2 | v_i), \cdots, p(v_n | v_i)]$ and 
$C$ is a constant irrelavant to $p(\cdot | v_i)$.
Accordingly, problem (\ref{obj_RLM}) becomes \cite{CAN}
\begin{equation}
    \label{obj_CAN}
    \min \limits_{p(\cdot | v_i)} \sum \limits_{i=1}^n \mathbb{E}_{v_j \sim p(\cdot | v_i)} d_{ij} + \gamma_i \|\bm p_i\|_2^2 ,
\end{equation}
where $\gamma_i$ controls the sparsity of each node. 
It can be proved that the solution has exact $k$ non-zero entries (\textit{i.e.,} $k$-sparse), 
if $\gamma_i = \frac{1}{2} (k d_{i\cdot}^{(k+1)} - \sum _{v=1}^k d_{i\cdot}^{(v)})$
where $d_{i\cdot}^{(v)}$ denote the $v$-th smallest value of $\{d_{ij}\}_{j=1}^n$.
Furthermore, it will results in a closed-form solution, \textit{i.e.},
\begin{equation}\label{eq_p}
    p(v_j | v_i) = \frac{(d_{i\cdot}^{(k+1)} - d_{ij})_+}{\sum_{v=1}^k d_{i\cdot}^{(k+1)} - d_{i\cdot}^{(v)}},
\end{equation}
where $(\cdot)_+ = \max(\cdot, 0)$. 
The details can be found in \cite{CAN}.
It should be emphasized that there are two reasons to utilize the above measurement.
On the one hand, 
the $k$-sparse property is preferable for graph-based clustering 
since the underlying probability of nodes from different clusters should approximate 0. 
On the other hand, there is only one hyper-parameter $k$ in it
in our model, which is much easier to tune.

\begin{algorithm}[t]
    \centering
    \caption{Algorithm to optimize AdaGAE}
    \label{alg}
    \begin{algorithmic}
        \REQUIRE Initial sparsity $k_0$, the increment of sparsity $t$ and number of iterations to update weight adjacency $T$. 
        \STATE $Z = X$, $k = k_0$.
        \FOR{$i = 1, 2, \cdots, T$}{
            \STATE Compute $p(\cdot | v_i)$ and $\hat A$ by solving problem (\ref{obj_CAN}) with $k$. 
            \REPEAT 
                \STATE Update GAE with Eq. (\ref{obj}) by the gradient descent.
            \UNTIL{convergence or exceeding maximum iterations.}
            \STATE Get new embedding $Z$.
            \STATE $k = k + t$.
        } 
        \ENDFOR
        \STATE Perform spectral clustering on $\hat A$, or run $k$-means on $Z$.
        \ENSURE Clustering assignments.
    \end{algorithmic}   
\end{algorithm}

\subsection{Graph Auto-Encoder for Weighted Graph}\label{subsection_architecture}
In the previous subsection, we only focus on how to obtain a sparse connectivity
distribution with the fixed $f(\cdot)$.
With the generative distribution, GNN can be easily extended into the non-graph data. 
Then in this subsection, we concentrate on how to obtain $f(\cdot)$ which is implemented by a graph auto-encoder with the estimated graph.
After obtaining the connectivity distribution by solving problem (\ref{obj_CAN}), 
we transform the directed graph to an undirected graph via 
$\mathcal W_{ij} = (p(v_i | v_j) + p(v_j | v_i)) / 2$, 
and the connectivity distribution serves as the reconstruction goal of graph auto-encoder. 
Firstly, we elaborate details of the specific graph auto-encoder.

\textbf{Encoder} ~
As shown in \cite{GCN}, graphs with self-loops show better performance, \textit{i.e.}, $\widetilde A = A + I$. Due to $d_{ii} = 0$, $p(v_i | v_i) \in (0, 1)$ if $k > 1$. Particularly, the weights of self-loops are learned adaptively rather than the primitive $I$. 
Consequently, we can simply set $\widetilde A = \mathcal W$ and 
$\hat A = \widetilde D^{-\frac{1}{2}} \widetilde A \widetilde D^{-\frac{1}{2}}$. 
The encoder consists of multiple GCN layers and aims to transform raw 
features to latent features with the constructed graph structure. 
Specifically speaking, the latent feature generated by $m$ layers is defined as 
\begin{equation}
    Z = \varphi_m(\hat A \varphi_{m-1} ( \cdots \varphi_1(\hat A X W_1) \cdots )W_m) .
\end{equation}

\begin{figure*}[t]
    \centering
    
    \subfigure[Raw features]{
        \label{subfigure_epoch_0}
        \includegraphics[width=0.19\linewidth]{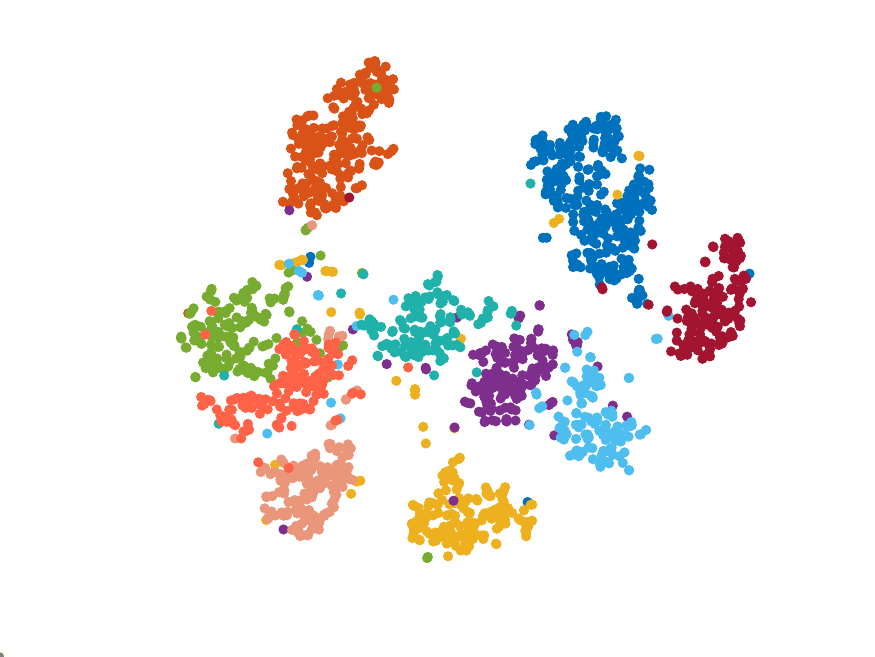}
    }
    \subfigure[Epoch-1]{
        \label{subfigure_epoch_1}
        \includegraphics[width=0.19\linewidth]{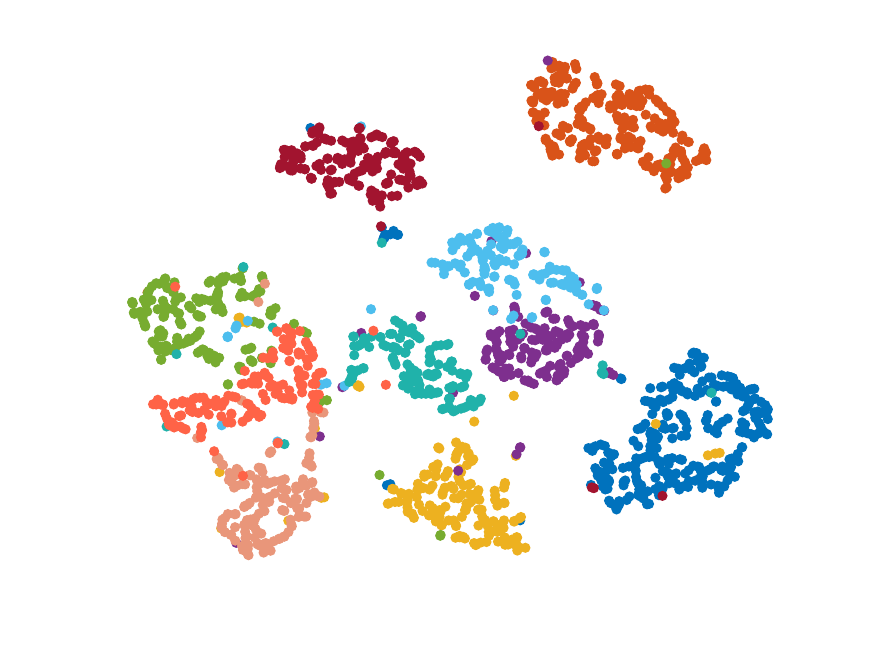}
    }
    \subfigure[Epoch-2]{\includegraphics[width=0.19\linewidth]{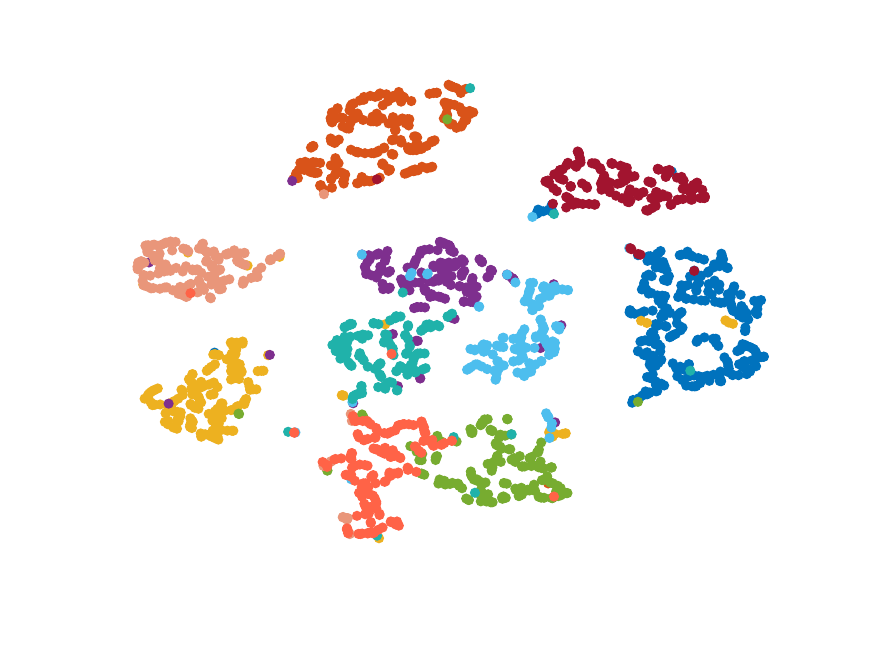}}
    \subfigure[Epoch-3]{\includegraphics[width=0.19\linewidth]{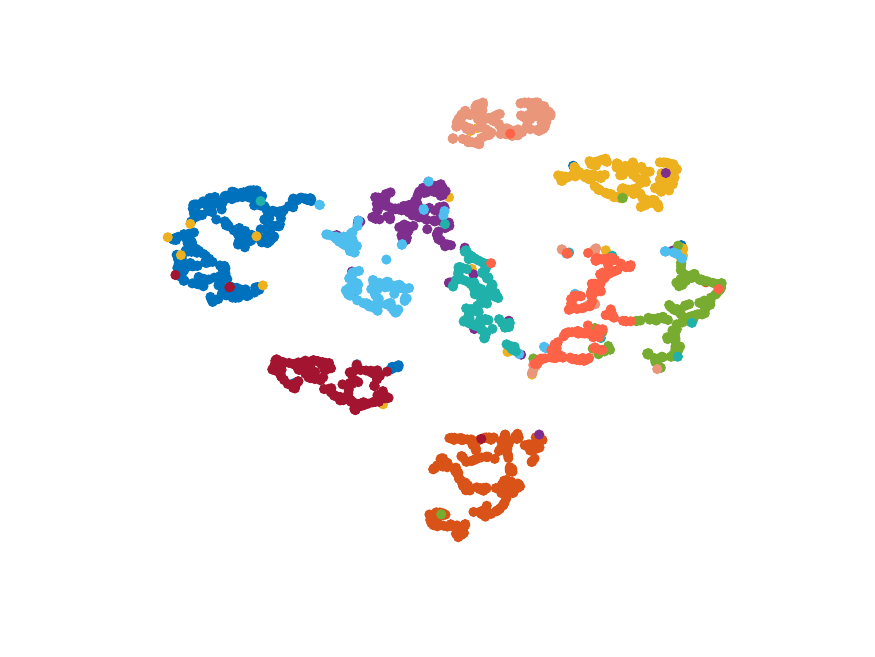}}
    \subfigure[Epoch-4]{\includegraphics[width=0.19\linewidth]{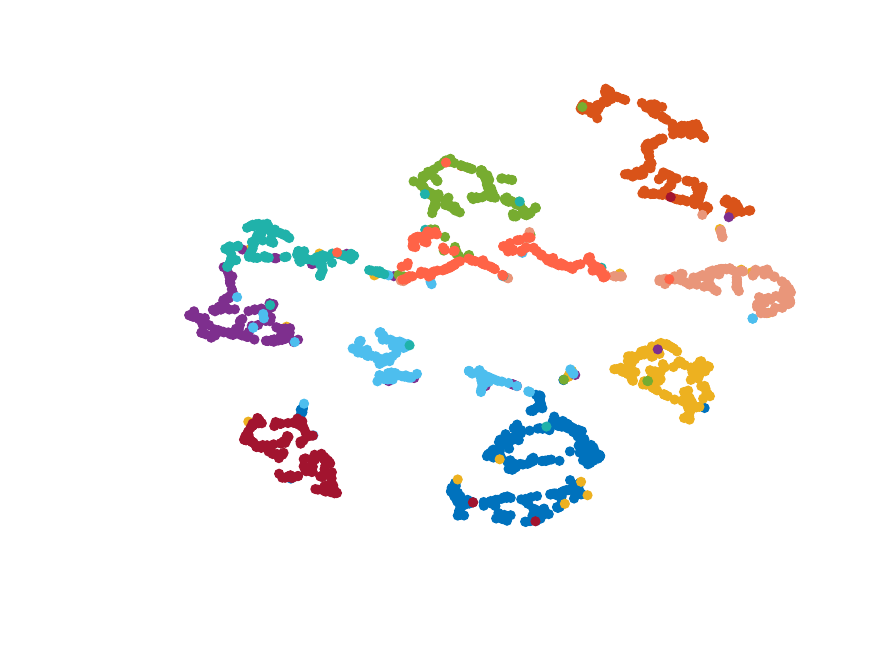}}

    \subfigure[Epoch-5]{\includegraphics[width=0.19\linewidth]{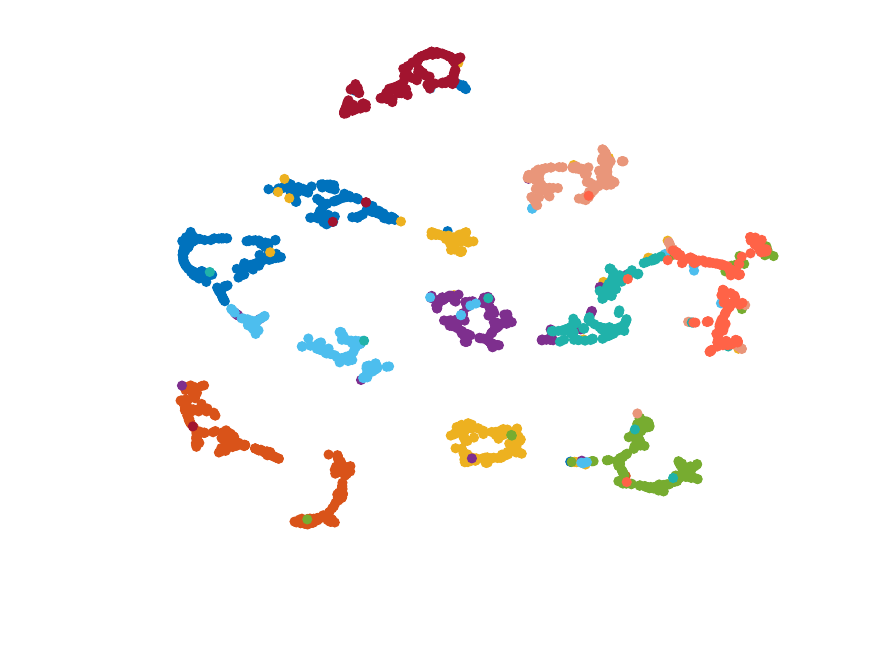}}
    \subfigure[Epoch-6]{\includegraphics[width=0.19\linewidth]{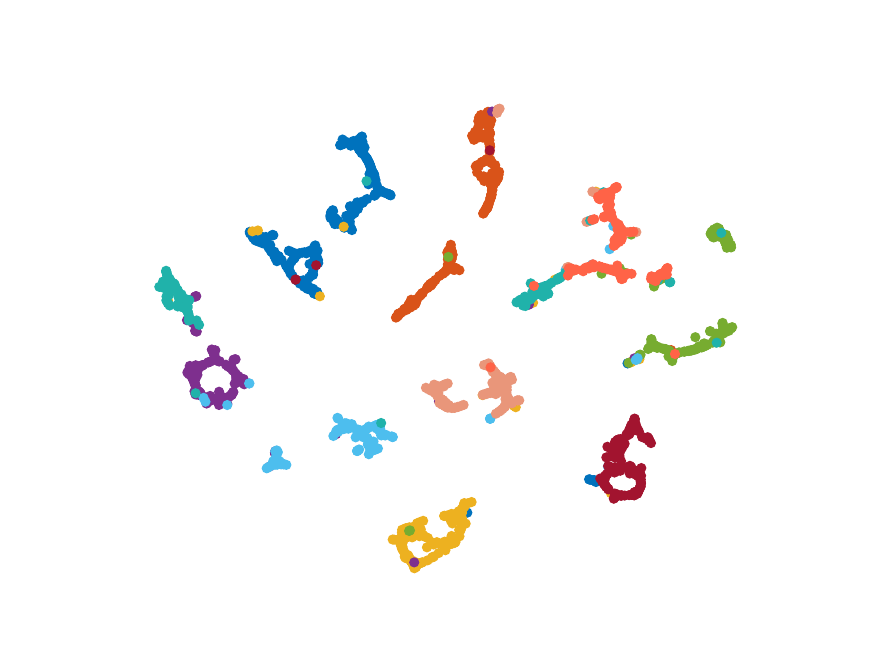}}
    \subfigure[Epoch-7]{\includegraphics[width=0.19\linewidth]{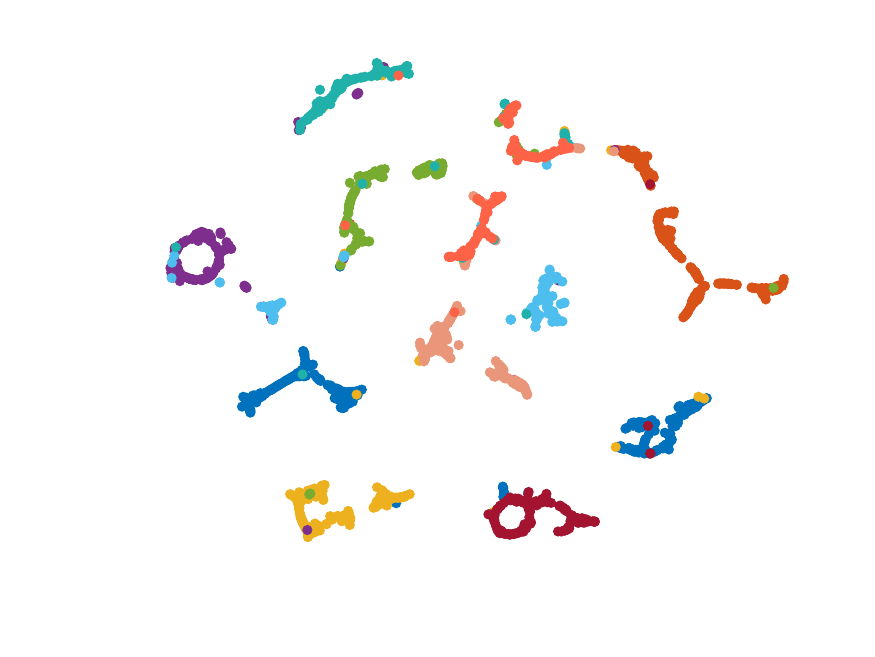}}
    \subfigure[Epoch-8]{
        \label{subfigure_epoch_8}
        \includegraphics[width=0.19\linewidth]{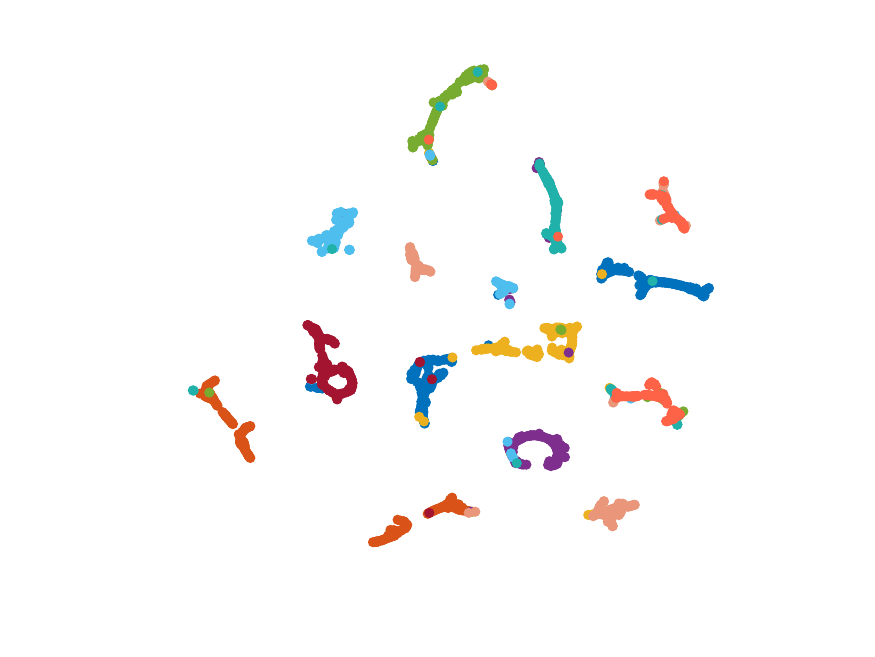}
    }
    \subfigure[Final embedding]{
        \label{subfigure_epoch_9}        
        \includegraphics[width=0.19\linewidth]{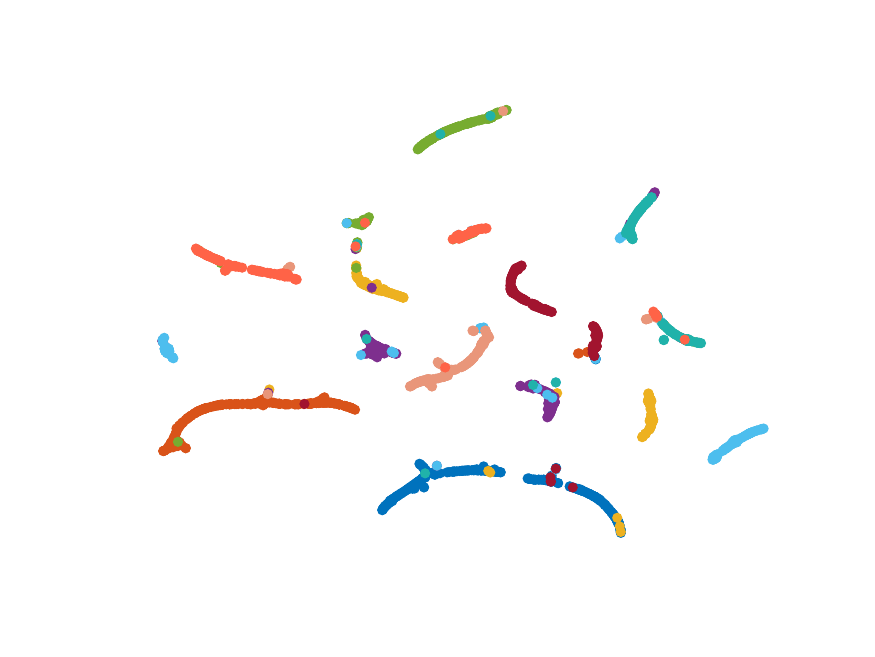}
    }
    \caption{Visualization of the learning process of AdaGAE on USPS. Figure \subref{subfigure_epoch_1}-\subref{subfigure_epoch_8} show the embedding learned by AdaGAE at the $i$-th epoch, while the raw features and the final results are shown in Figure \subref{subfigure_epoch_0} and \subref{subfigure_epoch_9}, respectively. An epoch corresponds to an update of the graph.}
    \label{figure_epoch}
\end{figure*}

\textbf{Decoder} ~
According to Assumption \ref{assumption_dist}, we aim to recover the connectivity distribution $p(v | v_i)$ based on Euclidean distances, 
instead of reconstructing the weight matrix $\widetilde A$ via inner-products. 
Firstly, distances of latent features $Z$ are calculated by $\hat d_{ij} = \|z_i - z_j\|_2^2$. Secondly, the connectivity distribution is reconstructed by a normalization step
\begin{equation}
    \label{define_q}
    q(v_j | v_i) = \frac{\exp({-\hat d_{ij}})}{\sum _{j = 1}^n \exp({-\hat d_{ij}})} .
\end{equation}
The above process can be regarded as inputting $-\hat d_{ij}$ into a SoftMax layer. Clearly, as $\hat d_{ij}$ is smaller, $q(v_{j} | v_i)$ is larger. In other words, the similarity is measured by Euclidean distances rather than inner-products, which are usually used in GAE. To measure the difference between two distributions, Kullback-Leibler (KL) Divergence is therefore utilized and the objective function is defined as 
\begin{equation}
    \label{obj_cross_entropy}
    \begin{split}
        \min \limits_{q(\cdot|v_i)} KL(p \| q) 
        \Leftrightarrow 
      \min \limits_{q(\cdot|v_i)} \sum \limits_{i, j = 1}^n p(v_j | v_i) \log \frac{1}{q(v_j | v_i)} .
    \end{split}
\end{equation}
Note that it is equivalent to minimize the cross entropy, which is widely employed in classification tasks.

\textbf{Loss} ~
Clearly, the representation provided by the graph auto-encoder can be used as 
a valid mapping, \textit{i.e.}, $f(\bm x_i) = \bm z_i$. Therefore, the loss that 
integrates the graph auto-encdoer and problem (\ref{obj_RLM}) is formulated as 
\begin{equation}
    \label{obj}
    \begin{split}
        & \min \limits_{q, Z} KL(p \| q) + \frac{\lambda}{2} \sum \limits_{i=1}^n \mathbb{E}_{p(v_j | v_i)} \|z_i - z_j\|_2^2 + \gamma_i \|\bm p_i\|_2^2\\
        \Leftrightarrow & \min \limits_{q, Z} \sum \limits_{i, j = 1}^n p(v_j | v_i) \log \frac{1}{q(v_j | v_i)} + \frac{\lambda}{2} \sum \limits_{i,j=1}^n \widetilde A_{ij} \|z_i - z_j\|_2^2 \\
        \Leftrightarrow & \min \limits_{q, Z} \sum \limits_{i, j = 1}^n p(v_j | v_i) \log \frac{1}{q(v_j | v_i)} + \lambda tr(Z \widetilde L Z^T) ,
    \end{split}
\end{equation}
where $\widetilde L = \widetilde D - \widetilde A$ and $\lambda$ is a tradeoff parameter to balance the cross entropy term and local consistency penalty term.

\begin{figure*}[t]
    \centering
    \subfigure[Raw features]{\includegraphics[width=0.23\linewidth]{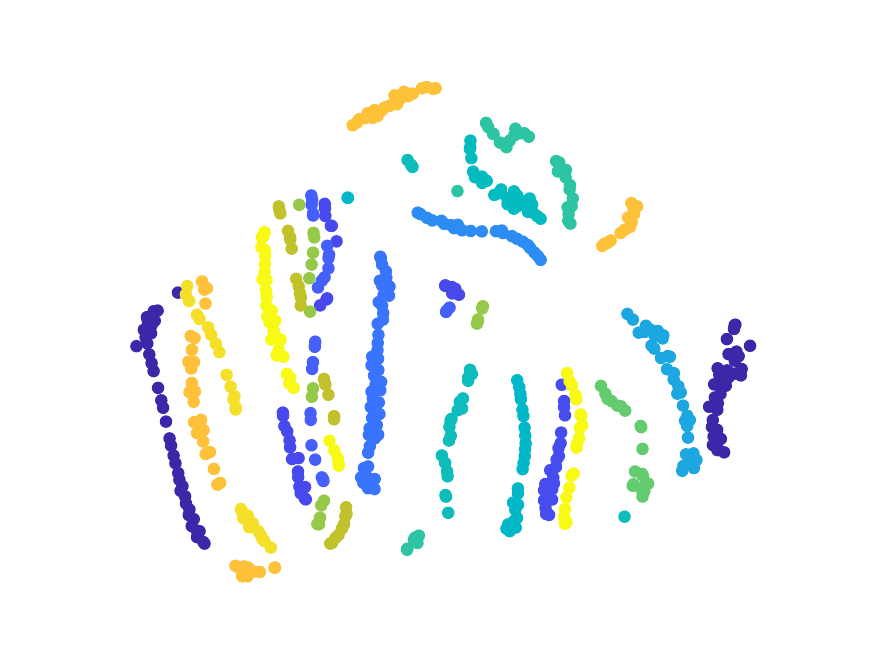}}
    \subfigure[AdaGAE with fixed $k$]{\label{subfigure_degeneration_1}\includegraphics[width=0.23\linewidth]{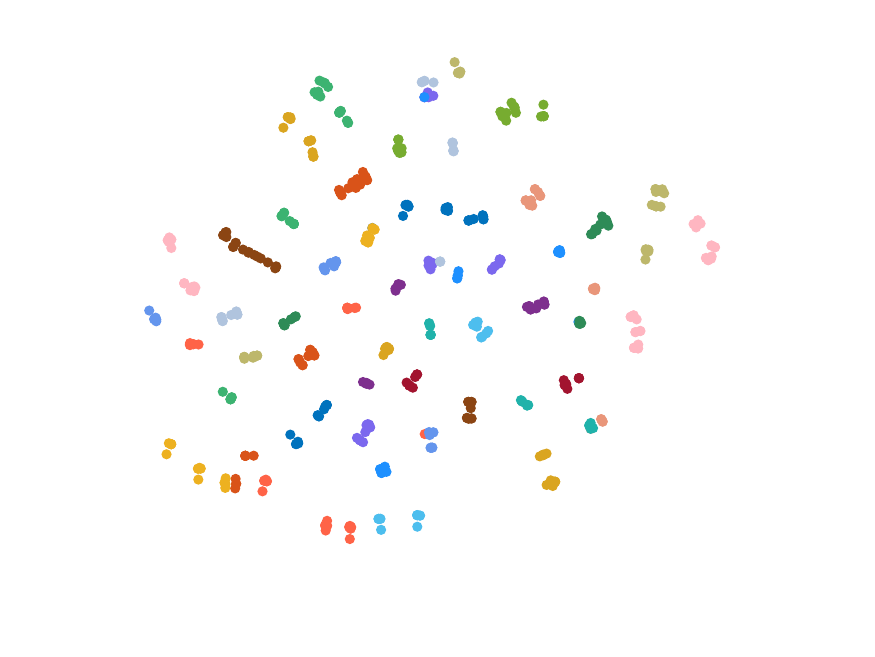}}
    \subfigure[GAE with fixed $\widetilde A$]{\includegraphics[width=0.23\linewidth]{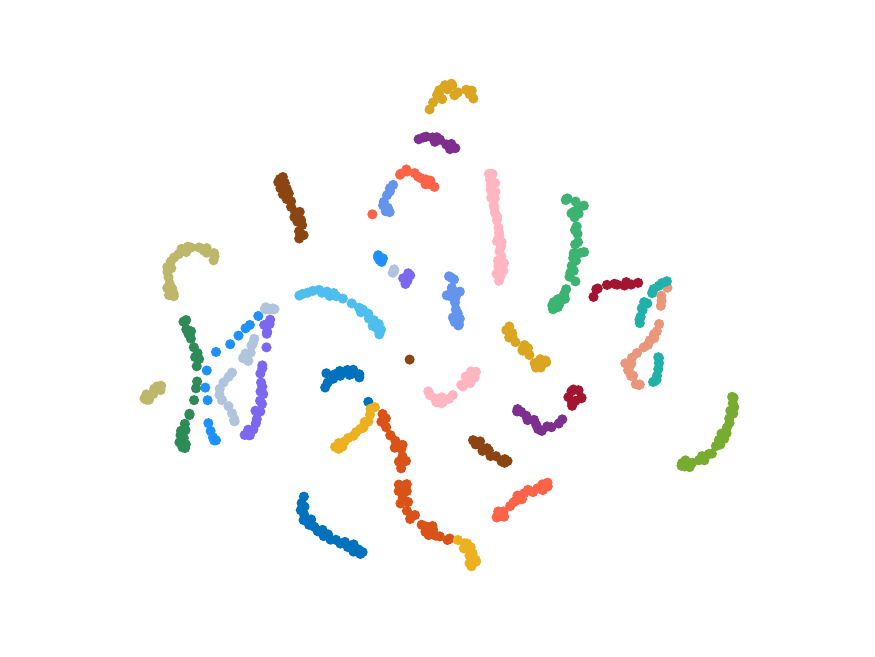}}
    \subfigure[AdaGAE]{\includegraphics[width=0.23\linewidth]{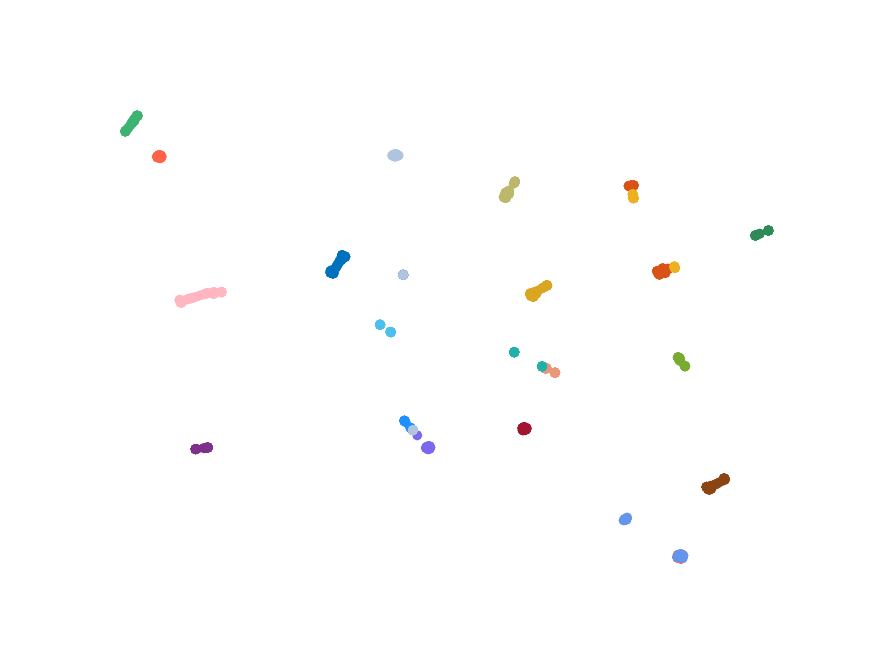}}
    
    \subfigure[Raw features]{\includegraphics[width=0.23\linewidth]{USPS-raw-eps-converted-to}}
    \subfigure[AdaGAE with fixed $k$]{\label{subfigure_degeneration_2}\includegraphics[width=0.23\linewidth]{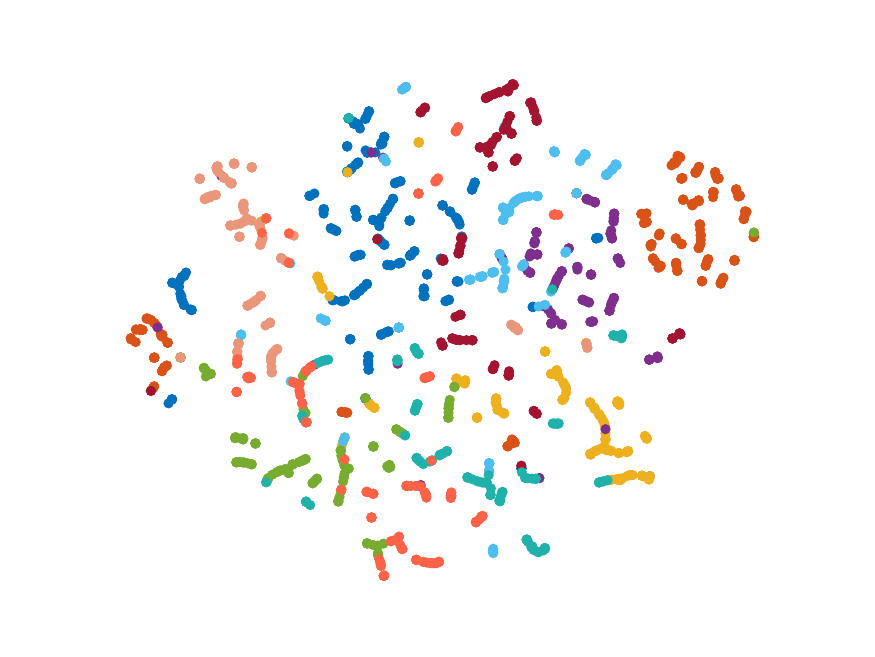}}
    \subfigure[GAE with fixed $\widetilde A$]{\includegraphics[width=0.23\linewidth]{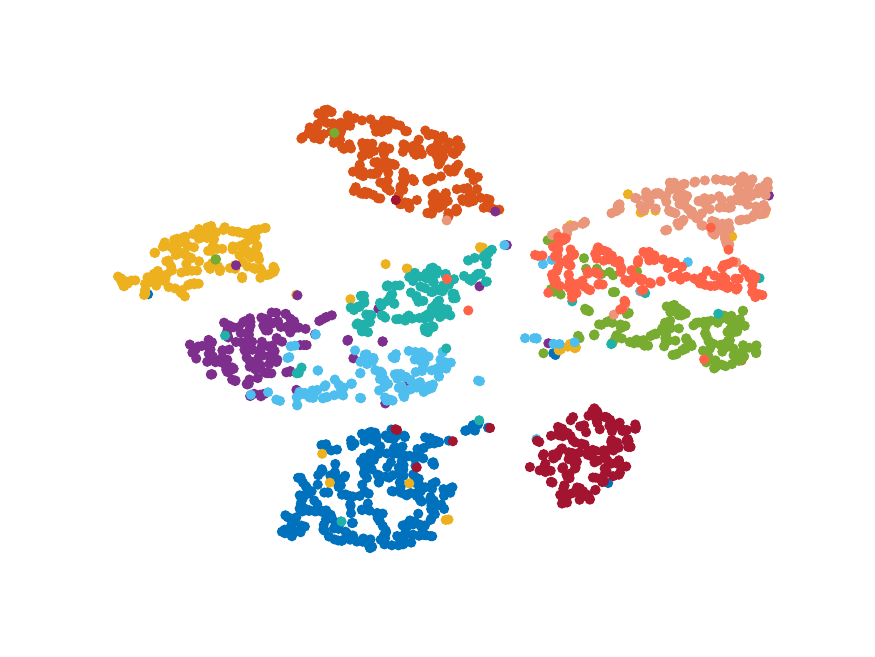}}
    \subfigure[AdaGAE]{\includegraphics[width=0.23\linewidth]{USPS-embedding-eps-converted-to}}

    \caption{t-SNE visualization on UMIST and USPS: The first and second line illustrate results on UMIST and USPS, respectively. 
                Clearly, AdaGAE projects most semblable samples into the analogous embedding. 
                Note that the cohesive embedding is preferable for clustering.}
    \label{figure_visualization}
\end{figure*}

\subsection{Adaptive Graph Auto-Encoder} \label{subsection_adaptive}
In the previous subsection, the weighted adjacency matrix is viewed as fixed during training. 
However, the weighted adjacency matrix is computed by optimizing Eq. (\ref{obj_CAN}) at first.
If the graph is not updated, then only the low-level information is utilized in GCN.  
The whole clustering process should contain connectivity learning and hence, the weighted adjacency should be updated adaptively during training. 
The adaptive evolution of the graph has two merits: 
1) The model is induced to exploit the high-level information behind the data. 
The high-level information exploitation can be also regarded as a promotion of the GAE with shallow architecture. 
2) It helps to correct the wrong links among samples that are caused by the low-level relationships. 

A feasible approach is to recompute the connectivity distribution based on the embedding $Z$, which contains the potential  manifold information of data. 
However, the following theorem shows that the simple update based on latent representations may lead to the collapse.

\begin{myTheorem}
    \label{theo_degeneration}
    Let $p^{(k)}(\cdot | v_i)$ be the $k$-largest $p(\cdot | v_i)$ and $\hat d_{ij} = \|z_i - z_j\|_2$ 
    where $z_i$ is generated by GAE with sparsity $k$. 
    Let $\hat p(\cdot | v_i) = \arg \min _{p(\cdot | v_i)} \sum _{i=1}^n \mathbb{E}_{v_j \sim p(\cdot | v_i)} \hat d_{i j} + \gamma_i \|\bm p_i\|_2^2$ where $\gamma_{i}$ controls the sparsity as $k$. 
    For a given constant $\mu \geq 0$, if $|p_j(\cdot | v_i) - q_j(\cdot | v_i)| \leq \varepsilon \leq \mu$ hold for any $i$, then $\forall j, l \leq k$,
    \begin{equation}
        | \hat p_j(\cdot | v_i) - \frac{1}{k}| \leq \mathcal{O}(\frac{\log (1 / \mu)}{\log (1 / \varepsilon)}) 
    \end{equation}
    Therefore, with $\varepsilon \rightarrow 0$, $\hat p(\cdot | v_i)$ will degenerate into a sparse and uniform distribution such that the weighted graph degenerates into an unweighted graph. 
\end{myTheorem}

Intuitively, the unweighted graph is indeed a bad choice for clustering. 
From another aspect, the degeneration to unweighted graph results in the appearance 
of diverse sub-clusters, which is illustrated in Figure \ref{figure_visualization}. 
Specifically, the perfect reconstruction indicates that 
some points from the identical class are projected into different sub-clusters.  
However, these tiny clusters scatter randomly in the latent space, which is shown in Figures \ref{subfigure_degeneration_1} and \ref{subfigure_degeneration_2}.
The phenomenon leads to the collapse of automatic update. 
Therefore, the update step with the same sparsity coefficient $k$ may result in collapse.  To address this problem, we assume that 
\begin{myAssumption}
    \label{assumption}
    Suppose that the sparse and weighted adjacency contains sufficient information. 
    Then, with latent representations, samples of an identical cluster become more cohesive measured by Euclidean distance.
\end{myAssumption}
According to the above assumption, samples from a cluster are more likely to lie in a local area after GAE mapping. 
The motivation to avoid the collapse is to build connection between those sub-clusters 
before they become thoroughly random in the latent space. 
Hence, the sparsity coefficient $k$ increases when updating weight sparsity. The step size $t$ which $k$ increases with needs to be discussed. In an ideal situation, we can define the upper bound of $k$ as 
\begin{equation}
    k_m^* = \min (| \mathcal{C}_1 |, | \mathcal{C}_2 |, \cdots, | \mathcal{C}_c |) ,
\end{equation}
where $\mathcal{C}_i$ denotes the $i$-th cluster and $|\mathcal{C}_i|$ is the size of $\mathcal{C}_i$. Although $|\mathcal C_i|$ is not known, we can define $k_m$ empirically to ensure $k_m \leq k_m^*$. For instance, $k_m$ can be set as $\lfloor \frac{n}{c} \rfloor$ or $\lfloor \frac{n}{2c} \rfloor$. Accordingly, the step size $t = \frac{k_m - k_0}{T}$ where $T$ is the number of iterations to update the weight adjacency. 

To sum up, Algorithm \ref{alg} summarizes the whole process to optimize AdaGAE.

\subsubsection{Another Explanation of Degeneration}
Theorem \ref{theo_entropy_equal} demonstrates that the SoftMax output layer 
with $-\hat d_{ij}$ is equivalent to solve problem (\ref{obj_RLM}) with a 
totally different implementation of $\mathcal{R}(\cdot, \cdot)$. 
\begin{myTheorem} \label{theo_entropy_equal}
    The decoder of AdaGAE is equivalent to solve the problem 
        $\min _{q(\cdot | v_i)} \sum _{i = 1}^n \mathbb E_{v_j \sim q(\cdot | v_i)} \hat d_{ij} + KL(q(\cdot | v_i) \| p(\cdot | v_i))$
    where $KL(q \| p) = \sum_i q_i \log \frac{q_i}{p_i}$ represents the KL-divergence of two discrete distributions.
\end{myTheorem}
Overall, with the fixed $k$, better reconstruction means the worse update.
Therefore, the perfect reconstruction may lead to bad performance.

\subsection{Spectral Analysis}
As mentioned in the above subsection, AdaGAE generates a weighted graph with adaptive self-loops. Analogous to SGC \cite{SimplifyingGCN}, Theorem \ref{theo_spectrum} shows that adaptive self-loops also reduce the spectrum of the normalized Laplacian, \textit{i.e.}, it smooths the Laplacian.
\begin{myTheorem} \label{theo_spectrum}
    Let $\widetilde A' = \widetilde A - diag(\widetilde A)$ and $\hat A' = \widetilde D'^{-\frac{1}{2}} \widetilde A' \widetilde D'^{-\frac{1}{2}}$. According to eigenvalue decomposition, suppose $I - \hat A = Q \Lambda Q^T$ and $I - \hat A' =  Q'\Lambda' Q'^T$. The following inequality always holds
    \begin{equation}
        0 = \lambda_1 = \lambda'_1 < \lambda_n < \lambda'_n ,
    \end{equation}
    where $\lambda_i$ and $\lambda_i'$ denote the $i$-th smallest eigenvalue of $I - \hat A$ and $I - \hat A'$, respectively.
\end{myTheorem}

\subsection{Computational Complexity}
In the phase of training GCN, the most time-consuming operation is to compute 
$\hat A T_i$ where $T_i = XW_i \in \mathbb R^{n \times d_i}$. 
Since $\hat A$ is sparse, the amount of non-zero entries is denoted by $|\mathcal{E}|$. 
Overall, the computational complexity of AdaGAE with $m$ layers is $O(l |\mathcal{E}| \sum_{i=1}^m d_i)$ 
where $l$ is the total number of gradient descent.
Therefore, the computational complexity of each iteration to update GCN is $O(|\mathcal{E}| d_i)$. 
To construct the graph matrix, $A$, $O(n^2)$ time is required which is same with the spectral clustering. 
After the embedding is obtained, the complexity to get clustering assignments is $O(n^2 c)$ (using the spectral clustering) or $O(ndc)$ (using $k$-means).

\begin{table*}[t]
    \centering
    \setlength{\abovecaptionskip}{0cm}
    \setlength{\tabcolsep}{2mm}
    \setlength{\belowcaptionskip}{-10mm}
    \renewcommand\arraystretch{1.1}
    \caption{ACC (\%)}
    \label{table_acc}
    \begin{tabular}{l c c c c c c c c c c }
    \hline
    
    \hline
        Methods & Text & 20news & Isolet & Segment & PALM & UMIST & JAFFE & COIL20 & USPS & MNIST \\
    \hline
    \hline
        K-Means & \underline{86.34} & 25.26 & 59.11 & 54.97 & 70.39 & 42.87 & 72.39 & 58.26 & 64.67 & 55.87 \\
        CAN \cite{CAN} & 50.31 & 25.39 & 61.47 & 49.13 & 88.10 & \underline{69.62} & \underline{96.71} & \underline{84.10} & 67.96 & 74.85 \\
        RCut \cite{RatioCut} & 53.44 & 28.06 & \underline{65.96} & 43.23 & 61.36 & 61.31 & 83.62 & 69.57 & 63.86 & 63.52 \\
        NCut \cite{SC} & 55.34 & 31.26 & 60.06 & 51.74 & 61.19 & 60.05 & 80.44 & 70.28 & 63.50 & 64.90 \\
        DEC \cite{DEC} & 50.62 & 25.11 & 34.17 & 14.29 & 27.45 & 36.47 & 62.95 & 74.35 & 42.30 & 81.22\\
        DFKM \cite{DFKM} & 52.77 & 29.65 & 51.99 & 51.47 & 67.45 & 45.47 & 90.83 & 60.21 & {73.42} & 48.37\\
        GAE \cite{GAE} & 53.45 & 25.59 & 61.41 & \underline{60.43} & \underline{88.45} & 61.91 & 94.37 & 69.10 & \underline{76.63} & 70.22\\
        MGAE \cite{MGAE} & 50.48 & \underline{41.47} & 46.31 & 50.44 & 51.47 & 49.19 & 87.22 & 60.99 & 64.13 & 55.17\\
        GALA \cite{GALA} & 50.31 & 28.16 & 53.59 & 49.57 & 79.45 & 41.39 & 94.37 & 80.00 & 67.64 & 74.26 \\
        SDCN \cite{SDCN} & 55.70 & 31.06 & 57.56 & 50.95 & 27.75 & 27.65 & 33.80 & 41.04 & 37.43 & 66.85 \\ 
        SpectralNet \cite{SpectralNet} & 48.20 & 25.37 & 52.98 & 51.39 & 88.30 & 52.53 & 88.26 & 75.63 & 70.93 & \underline{82.03} \\
    \hline
        GAE$^\dag$  & 50.31 & 33.55 & 62.05 & 47.66 & 82.10 & 72.17 & 96.71 & 85.97 & 79.40 & 71.07 \\
        Method-A & 50.00 & 38.54 & 66.15 & 41.13 & 88.30 & 73.22 & 96.71 & 92.43 & 67.48 & 73.87 \\
        Method-B & 51.13 & 33.35 & 54.49 & 38.66 & 91.80 & 32.00 & 47.42 & 33.82 & 34.09 & 14.04 \\
        AdaGAE & \textbf{89.31} & \textbf{77.28} & \textbf{66.22} & \textbf{60.95} & \textbf{95.25}& \textbf{83.48} & \textbf{97.27} & \textbf{93.75} & \textbf{91.96} & \textbf{92.88} \\
    \hline
    
    \hline
    \end{tabular}
    
\end{table*}

\section{Experiments} \label{section_experiments}
In this section, details of AdaGAE are demonstrated and the results are shown. The visualization supports the theoretical analyses mentioned in the previous section. 

\begin{table}[h]
    \centering
    \setlength{\tabcolsep}{3mm}
    \setlength{\abovecaptionskip}{0cm}
    \setlength{\belowcaptionskip}{-10mm}
    \caption{Information of Datasets}
    \begin{tabular}{l c c c}
    \hline
    
    \hline
        Name & \# Features & \# Size & \# Classes \\
    \hline
    \hline
        Text & 7511 & 1946 & 2 \\
        20news & 8014 & 3970 & 4 \\
        Isolet & 617 & 1560 & 26 \\
        Segment & 19 & 2310 & 7 \\
        PALM & 256 & 2000 & 100 \\
        UMIST & 1024 & 575 & 20 \\
        JAFFE & 1024 & 213 & 10 \\
        COIL20 & 1024 & 1440 & 20 \\
        USPS & 256 & 1854 & 10 \\
        MNIST & 784 & 10000 & 10 \\
    \hline
    \end{tabular}
    \label{table_datasets}
\end{table}

\subsection{Datasets and Compared Methods}
AdaGAE is evaluated on 10 datasets of different types, 
including 2 text datasets (\textit{Text} and \textit{20news}), 
3 UCI \cite{UCI} datasets (\textit{Isolet}, \textit{Segment}, and \textit{PALM}), 
and 5 image datasets (\textit{UMIST} \cite{UMIST}, \textit{JAFFE} \cite{JAFFE}, \textit{COIL20} \cite{COIL20}, \textit{USPS} \cite{USPS}, and \textit{MNIST-test} \cite{MNIST}). 
Note that USPS used in our experiments is a subset with 1854 samples of the whole dataset. 
To keep notations simple, MNIST-test is denoted by \textit{MNIST}. 
Note that all features are rescaled to $[0, 1]$. 
The details of these datasets are shown in Table \ref{table_datasets}.

To evaluate the performance of AdaGAE, 11 methods serve as competitors. 
To ensure the fairness, 4 clustering methods without neural networks are used, including \textit{K-Means}, \textit{CAN} \cite{CAN}, Ratio Cut (\textit{RCut}) \cite{RatioCut}, and Normalized Cut (\textit{NCut}) \cite{SC}. 
Three deep clustering methods for general data, \textit{DEC} \cite{DEC} \textit{DFKM} \cite{DFKM}, and \textit{SpectralNet} \cite{SpectralNet}, also serve as an important baseline. Besides, four GAE-based methods are used, including \textit{GAE} \cite{GAE}, \textit{MGAE} \cite{MGAE}, \textit{GALA} \cite{GALA}, and \textit{SDCN} \cite{SDCN}. All codes are downloaded from the homepages of authors.

\subsection{Experimental Setup}
In our experiments, the encoder consists of two GCN layers. If the input dimension is 1024, the first layer has 256 neurons and the second layer has 64 neurons. Otherwise, the two layers have 128 neurons and 64 neurons respectively. The activation function of the first layer is set as ReLU while the other one employs the linear function. 
The initial sparsity $k_0$ is set as 5 and the upper bound $k_m$ is searched from $\{\lfloor \frac{n}{c}\rfloor, \lfloor \frac{n}{2c}\rfloor\}$. The tradeoff coefficient $\lambda$ is searched from $\{10^{-3}, 10^{-2}, \cdots, 10^3\}$. The number of graph update step is set as 10 and the maximum iterations to optimize GAE varies in $[150, 200]$. 
To use Ratio Cut and Normalized Cut, we construct the graph via Gaussian kernel, which is given as 
\begin{equation}
    s_{ij} = \frac{\exp(-\frac{\|\bm x_i - \bm x_j\|_2^2}{\sigma})}{\sum _{j \in \mathcal N_i}\exp(-\frac{\| \bm x_i - \bm x_j\|_2^2}{\sigma})},
\end{equation}
where $\mathcal N_i$ represents $m$-nearest neighbors of sample $\bm x_i$. 
$m$ is searched from $\{5, 10\}$ and 
$\sigma$ is searched from $\{10^{-3}, 10^{-2}, \cdots, 10^3\}$. 
The maximum iterations of GAE with fixed $\widetilde A$ is set as 200.

\begin{figure}[t]
    \centering
    \subfigure[UMIST]{\includegraphics[width=0.48\linewidth]{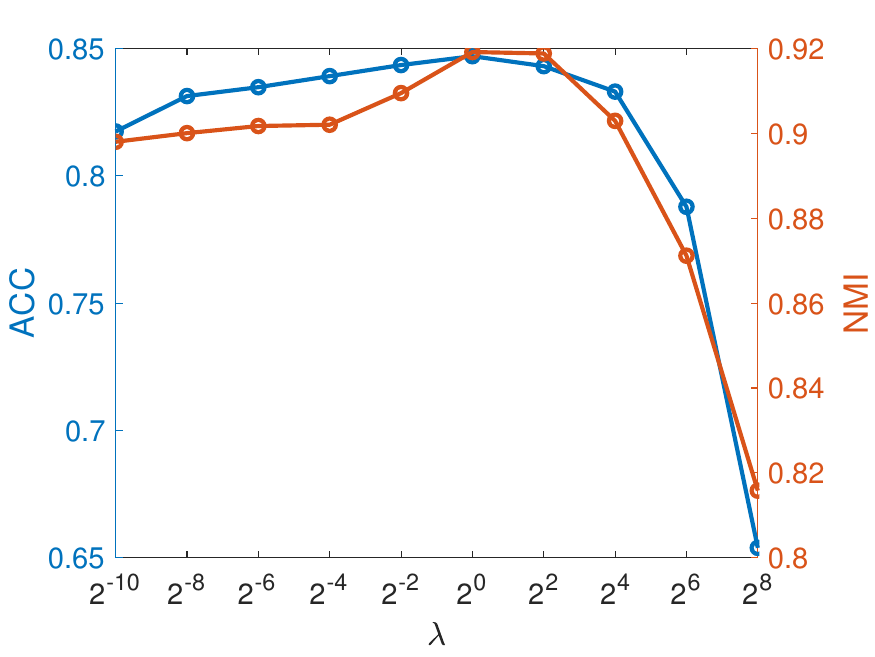}}
    \subfigure[USPS]{\includegraphics[width=0.48\linewidth]{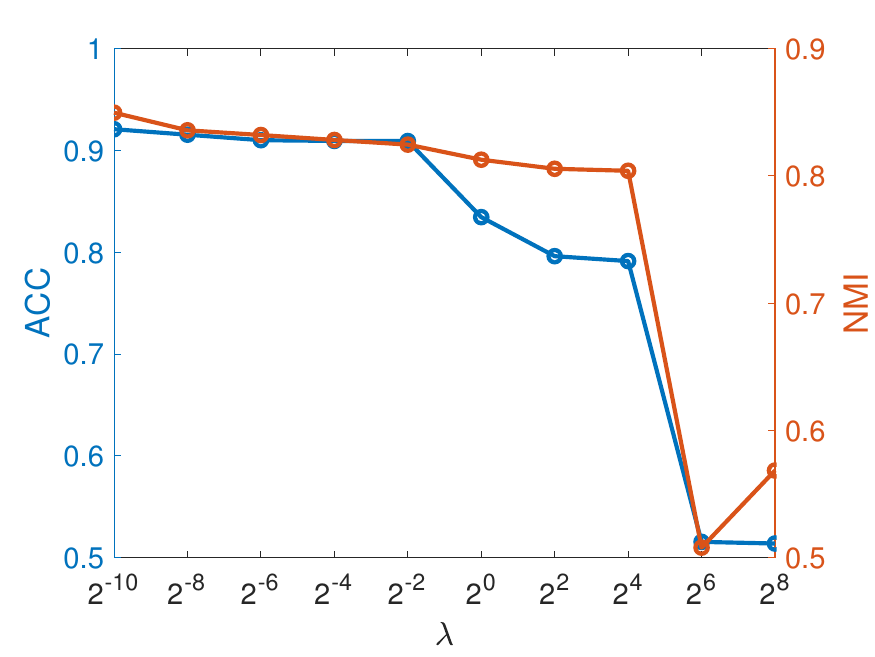}}

    \caption{Parameter sensitivity of $\lambda$ on UMIST and USPS. 
    On UMIST, the second term improves the performance distinctly. Besides, if $\lambda$ is not too large, AdaGAE will obtain good results. }
    \label{figure_sensitivity}
\end{figure}

To study the roles of different parts, the ablation experiments are conducted: GAE with $\lambda=0$ and fixed $\widetilde A$  (denoted by $GAE^\dag$), AdaGAE with fixed $\widetilde A$ (denoted by \textit{Method-A}), and AdaGAE with fixed sparsity $k$ (denoted by \textit{Method-B}). 

Two popular clustering metrics, the clustering accuracy (\textit{ACC}) and normalized mutual information (\textit{NMI}), are employed to evaluate the performance. All methods are run 10 times and the means are reported. 
The code of AdaGAE is implemented under pytorch-1.3.1 on a PC with an NVIDIA GeForce GTX 1660 GPU. 
The exact settings are shown in Table \ref{table_details}. 

\begin{table*}[t]
    \centering
    \setlength{\abovecaptionskip}{0cm}
    \setlength{\belowcaptionskip}{-10mm}
    \renewcommand\arraystretch{1.1}
    \setlength{\tabcolsep}{2mm}
    \caption{NMI (\%)}
    \label{table_nmi}
    \begin{tabular}{l c c c c c c c c c c }
    \hline
    
    \hline
        Methods & Text & 20news & Isolet & Segment & PALM & UMIST & JAFFE & COIL20 & USPS & MNIST \\
    \hline
    \hline
    
        K-Means & \underline{51.09} & 0.27 & 74.15 & 55.40 & 89.98 & 65.47 & 80.90 & 74.58 & 62.88 & 54.17 \\
        CAN \cite{CAN} & 2.09 & 3.41 & 78.17 & 52.62 & \underline{97.08} & \underline{87.75} & \underline{96.39} & \underline{90.93} & \underline{78.85} & 77.30 \\
        RCut \cite{RatioCut} & 0.35 & 4.59 & \underline{78.55} & 53.22 & 85.78 & 77.64 & 90.63 & 84.16 & 70.35 & 70.29 \\
        NCut \cite{SC} & 0.93 & 4.57 & 72.11 & 51.11 & 85.21 & 77.54 & 89.56 & 84.70 & 70.43 & 72.15 \\
        DEC \cite{DEC} & 2.09 & 0.27 & 70.13 & 0.00 & 55.22 & 56.96 & 82.83 & 90.37 & 48.71 & \underline{80.25} \\
        DFKM \cite{DFKM} & 0.25 & 3.39 & 69.81 & 49.91 & 86.74 & 67.04 & 92.01 & 76.81 & 71.58 & 38.75\\
        GAE \cite{GAE} & 35.63 & 4.49 & 74.63 & 50.22 & 94.87 & 80.24 & 93.26 & 86.45 & 76.02 & 65.58 \\
        MGAE \cite{MGAE} & 1.40 & \underline{23.41} & 65.06 & 39.26 & 81.13 & 68.00 & 89.65 & 73.59 & 62.18 & 57.33 \\
        GALA \cite{GALA} & 2.09 & 3.09 & 71.60 & \underline{57.09} & 89.50 & 63.71 & 92.54 & 87.71 & 71.50 & 75.65 \\
        SDCN \cite{SDCN} & 0.94 & 1.71 & 73.05 & 45.89 & 60.18 & 38.65 & 46.85 & 60.07 & 38.97 & 69.14 \\ 
        SpectralNet \cite{SpectralNet} & 3.77 & 3.45 & 78.60 & 53.64 & 92.87 & 79.99 & 87.02 & 88.00 & 74.84 & \underline{81.66} \\
    \hline
        GAE$^\dag$ & 2.09 & 10.30 & 78.61 & 53.86 & 93.86 & 87.00 & 96.39 & 95.62 & 78.46 & 76.15 \\
        Method-A & 0.52 & 25.18 & \textbf{79.76} & 52.16 & 96.96 & 87.04 & 96.39 & 97.26 & 76.45 & 76.59\\
        Method-B & 0.14 & 6.38 & 77.70 & 30.02 & 97.80 & 52.08 & 59.55 & 55.46 & 35.39 & 1.63 \\
        AdaGAE & \textbf{51.59} & \textbf{49.60} & 78.89 & \textbf{59.76} & \textbf{98.18} & \textbf{91.03} & \textbf{96.78} & \textbf{98.36} & \textbf{84.81} & \textbf{85.31} \\
    \hline
    
    \hline
    \end{tabular}
    
    \end{table*}

\begin{table}[t]
    \centering
    \renewcommand\arraystretch{1.1}
    \caption{$k_0$: initial sparsity; $\gamma$: learning rate; $t_i$: number of iterations to update GAE; struct: neurons of each layer used in AdaGAE.}
    \label{table_details}
    \begin{tabular}{l c c c c c c c}
    \hline
    
    \hline
        Dataset & $\lambda$ & $k_0$ & $\gamma$  & $t_i$ & $k_m$ & $T$ & struct \\
    \hline
    \hline
        Text & 0.01 & 30 & $5\times10^{-3}$ & 150 & $\lfloor \frac{n}{c} \rfloor$ & 10 & $d$-256-64 \\
        20news & 0.1 & 20 & $10^{-3}$ & 200 & $\lfloor \frac{n}{2c} \rfloor$ & 10 & $d$-256-64 \\
        Isolet & 0.1 & 20 & $10^{-3}$ & 200 & $\lfloor \frac{n}{c} \rfloor$ & 5 & $d$-256-64 \\
        PALM & $10$ & 10 & $10^{-3}$ &  50 & $\lfloor \frac{n}{c} \rfloor$ & 10 & $d$-256-64\\
        Segment & 100 & 20 & $10^{0}$ & 100 & (30) & 10 & $d$-10-7 \\
        UMIST & 1 & 5 & $10^{-3}$  & 50 & $\lfloor \frac{n}{c} \rfloor$ & 10 & $d$-256-64\\
        COIL20 & 1 & 5 & $10^{-2}$  & 100 & $\lfloor \frac{n}{2c} \rfloor$ & 10 & $d$-256-64\\
        JAFFE & $10^{-3}$ & 5 & $10^{-2}$  & 20 & $\lfloor \frac{n}{c} \rfloor$ & 10 & $d$-256-64\\
        USPS & $10^{-2}$ & 5 & $5 \times 10^{-3}$  & 150 & $\lfloor \frac{n}{c} \rfloor$ & 10 & $d$-128-64\\
        MNIST & $10^{-2}$ & 5 & $10^{-3}$ & 200 & $\lfloor \frac{n}{2c} \rfloor$ & 10 & $d$-256-64\\
    \hline
    \end{tabular}
\end{table}

\subsection{Experimental Results}
To illustrate the process of AdaGAE, Figure \ref{figure_epoch} shows the learned embedding on USPS at the $i$-th epoch. An epoch means a complete training of GAE and an update of the graph. The maximum number of epochs, $T$, is set as 10. In other words, the graph is updated 10 times. Clearly, the embedding becomes more cohesive with the update.

ACCs and NMIs of all methods are reported in Table \ref{table_acc} and \ref{table_nmi}. The best results of both competitors and AdaGAEs are highlighted in boldface while the suboptimal results are underlined. From Table \ref{table_acc} and \ref{table_nmi}, we conclude that: 





\begin{itemize}
    \item On text datasets (Text and 20news), most graph-based methods get a trivial result, as they group all samples into the same cluster such that NMIs approximate 0. Only $k$-means, MGAE, and AdaGAE obtain the non-trivial assignments. 
    \item Classical clustering models work poorly on large scale datasets. Instead, DEC and SpectralNet work better on the large scale datasets. Although GAE-based models (GAE, MGAE, and GALA) achieve impressive results on graph type datasets, they fail on the general datasets, which is probably caused by the fact that the graph is constructed by an algorithm rather than prior information. If the graph is not updated, the contained information is low-level. The adaptive learning will induce the model to exploit the high-level information. In particular, AdaGAE is stable on all datasets. 
    \item When the sparsity $k$ keeps fixed, AdaGAE collapses on most of datasets. For example, ACC shrinks about 50\% and NMI shrinks about 40\% on COIL20.
    \item From the comparison of 3 extra experiments, we confirm that the adaptive graph update plays a positive role. Besides, the novel architecture with weighted graph improves the performance on most of datasets. 
\end{itemize}
To study the impact of different parts of the loss in Eq. (\ref{obj}), the performance with different $\lambda$ is reported in Figure \ref{figure_sensitivity}.
From it, we find that the second term (corresponding to problem (\ref{obj_CAN})) plays an important role especially on UMIST. If $\lambda$ is set as a large value, we may get the trivial embedding according to the constructed graph. AdaGAE will obtain good results when $\lambda$ is not too large.

Besides, Figure \ref{figure_visualization} illustrates the learned embedding vividly. 
Combining with Theorem \ref{theo_degeneration}, if $k$ is fixed, 
then $\widetilde A$ degenerates into an unweighted adjacency matrix 
and a cluster is broken into a mass of groups. 
Each group only contains a small number of data points 
and they scatter chaotically, which leads to collapse. 
Instead, \textit{the adaptive process connects these groups 
before degeneration via increasing sparsity $k$, 
and hence, the embeddings in a cluster become cohesive}. 
Note that \textbf{the cohesive representations are preferable (not over-fitting) for clustering due to the unsupervised setting}.
It should be emphasized that a large $k_0$ frequently leads to capture the wrong information. 
After the transformation of GAE, the nearest neighbors are more likely to belong with the same cluster 
and thus it is rational to increasing $k$ with an adequate step size.

\section{Conclusion}
In this paper, we propose a novel clustering model for general data clustering, 
namely Adaptive Graph Auto-Encoder (\textit{AdaGAE}). 
A generative graph representation model is utilized to construct a weighted 
graph for GAE. 
To exploit potential information, we employ a graph auto-encoder to exploit
the high-level information. 
As the graph used in GAE is constructed artificially, 
an adaptive update step is developed to update the graph with the help of 
learned embedding. 
Related theoretical analyses demonstrate the reason why AdaGAE with fixed 
sparsity collapses in the update step such that a dynamic sparsity is essential. 
In experiments, we show the significant performance of AdaGAE and verify 
the effectiveness of the adaptive update step via the ablation experiments. 
Surprisingly, the visualization supports the theoretical analysis well and 
helps to understand how AdaGAE works. 


\section{Proofs} \label{section_proof}

\subsection{Proof of Theorem \ref{theo_degeneration}}
\begin{proof}
    Without loss of generality, we focus on the connectivity distribution of node $v$ 
    and suppose that 
    $p(v_1 | v) \geq p(v_2 | v) \geq \cdots \geq p(u_k | v) > 0 = p(u_{k+1} | v) = \cdots = p(u_m | v)$. 
    Let $p_i = p(v_i | v)$, $q_i = q(v_i | v)$, $d_i = \|\bm z_i - \bm z\|$, then $q_i$ can be formulated as 
    \begin{equation}
        \begin{aligned}
            q_i = \frac{\exp(- d_i)}{\sum_{j=1}^m \exp(- d_j)} .
        \end{aligned}
    \end{equation}
    If for any $i$, $|p_i - q_i| \leq \varepsilon$always holds. 
    Apparently, $p_{k+1} = 0$. 
    Assume that $q_{k+1} = \tau \leq \varepsilon$ and we have 
    \begin{equation}
        \begin{split}
            & \frac{\exp(- d_{k+1})}{\sum_{j=1}^m \exp(- d_j)} = \tau 
            \Leftrightarrow ~  d_{k+1} = \log \frac{1}{\tau} - \log C .
        \end{split}
    \end{equation}
    where $C = \sum_{j=1}^m \exp(- d_j)$. 
    With the condition $p_k \geq \mu$, we have 
    \begin{equation}
        \begin{aligned}
            & \frac{\exp(- d_{k})}{\sum_{j=1}^m \exp(- d_j)} \geq p_k - \varepsilon \geq \mu - \varepsilon \\
            \Rightarrow ~ &  d_k \leq -\log C - \log (\mu - \varepsilon) . \\
        \end{aligned}
    \end{equation}
    Similarly, due to $p_1 \geq \frac{1}{k}$, 
    \begin{equation}
        \begin{split}
            & \frac{\exp(- d_{1})}{\sum_{j=1}^m \exp(- d_j)} \leq 1 
            \Rightarrow ~  d_1 \geq - \log C .\\
        \end{split}
    \end{equation}
    If update the connectivity distribution according to $\{ d_{i}\}_{i = 1}^m$, then for any $i \leq k$, 
    \begin{equation}
        \hat p_i = \frac{ d_{k+1} -  d_i}{\sum _{j=1}^k ( d_{k+1} -  d_j)} .
    \end{equation}
    We further have 
    \begin{align*}
        | \hat p_i -  \hat p_j| & = \frac{| d_j -  d_i|}{\sum _{j=1}^k ( d_{k+1} -  d_j)} \\
        & \leq \frac{| d_k -  d_1|}{\sum _{j=1}^k ( d_{k+1} -  d_j)} 
        = \frac{-\log (\mu - \varepsilon)}{\sum _{j=1}^k (\log \frac{1}{\tau} - \log C -  d_j)} \\
        & \leq \frac{-\log (\mu - \varepsilon)}{\sum _{j=1}^k (\log \frac{1}{\tau} - \log C -  d_k)} \\
        & \leq \frac{-\log (\mu - \varepsilon)}{k (\log(\mu - \varepsilon) - \log \tau)} \\
        & = \frac{1}{k} \cdot \frac{\log (\mu - \varepsilon)}{ \log \tau - \log(\mu - \varepsilon)} 
        = \frac{1}{k} \cdot \frac{1}{\frac{\log (1/\tau)}{\log (1/(\mu - \varepsilon))} - 1} \\
        & \leq \frac{1}{k} \cdot \frac{1}{\frac{\log (1 / \varepsilon)}{\log 1 / (\mu - \varepsilon)} - 1} 
        = \mathcal{O}(\frac{\log (1 / \mu)}{\log (1 / \varepsilon)}) . 
    \end{align*}
    Accordingly, for any $i \leq k$, we have the following results
    \begin{align*}
        |\hat p_i - \frac{1}{k}| & \leq |\hat p_i - \frac{1}{k} \sum _{j=1}^k \hat p_j| 
        \leq \frac{1}{k} |\sum _{j=1}^k (\hat p_i - \hat p_j)| \\
        & \leq \frac{1}{k} \sum _{j=1}^k | (\hat p_i - \hat p_j) | 
        \leq \mathcal{O}(\frac{\log (1 / \mu)}{\log (1 / \varepsilon)}) . 
    \end{align*}
    The proof is easy to extend to the other nodes. Hence, the theorem is proved. 
\end{proof}

\subsection{Proof of Theorem \ref{theo_entropy_equal}}

\begin{proof}
    According to the definition of KL-divergence, 
    \begin{equation}
        KL(q(\cdot | v_i) \| p(\cdot | v_i)) = \sum _{j=1}^n q(v_j | v_i) \log (\frac{q(v_j | v_i)}{p(v_j | v_i)}) ,
    \end{equation}
    the problem is equivalent to the following $i$-th subproblem  
    \begin{equation}
        \begin{split}
            \min \limits_{q_{ij}} & \sum \limits_{j = 1}^n q_{ij} \hat d_{ij} + q_{ij} \log q_{ij} , 
            ~ s.t. \sum \limits_{j=1}^n q_{ij} = 1, q_{ij} > 0 .
        \end{split} 
    \end{equation}
    Similarly, the subscript $i$ is omitted to keep notations uncluttered. The Lagrangian is 
    \begin{equation}
        \mathcal L = \sum \limits_{j = 1}^n q_{j} \hat d_{j} + q_{j} \log q_{j} + \alpha (1 - \sum \limits_{j=1}^n q_j) + \sum \limits_{j=1} \beta_j (-q_j) .
    \end{equation}
    Then the KKT conditions are 
    \begin{equation}
            \hat d_j + 1 + \log q_j - \alpha - \beta_j = 0,
            \sum \limits_{j=1}^n q_j = 1, 
            \beta_j q_j = 0, 
            \beta_j \geq 0 .
    \end{equation}
    Due to $q_j > 0$, $\beta_j = 0$. Use the first one, we have 
    \begin{equation}
        q_j = \exp(\alpha - \hat d_j - 1) .
    \end{equation}
    Combine it with the second line and we have 
    \begin{equation}
        \exp(\alpha) \sum _{j=1}^n \exp(-\hat d_j - 1) = 1 .
    \end{equation}
    Furthermore, we have 
    \begin{equation}
        q_j = \frac{\exp(-\hat d_j - 1)}{\sum _{j=1}^n \exp(-\hat d_j - 1)} = \frac{\exp(-\hat d_j)}{\sum _{j=1}^n \exp(-\hat d_j)} .
    \end{equation}
    With $\hat d_{ij} = \|z_i - z_j\|_2$, the theorem is proved.
\end{proof}

\subsection{Proof of Theorem \ref{theo_spectrum}}
It should be pointed out that the proof imitates the corresponding proof in \cite{SimplifyingGCN}. Analogous to Lemma 3 in \cite{SimplifyingGCN}, we first give the following lemma without proof.
\begin{myLemma}
    Let $\alpha_1 \leq \alpha_2 \leq \cdots \leq \alpha_n$ be eigenvalues of $\hat D^{-\frac{1}{2}} \hat A' \hat D^{-\frac{1}{2}}$ and $\beta_1 \leq \beta_2 \leq \cdots \leq \beta_n$ be eigenvalues of $\hat D'^{-\frac{1}{2}} \hat A' \hat D'^{-\frac{1}{2}}$. The following inequality always holds
    \begin{equation}
        \alpha_1 \geq \frac{\beta_1}{1 + \min \frac{\hat A_{ii}}{\hat D'_{ii}}},
        \alpha_n \leq \frac{1}{1 + \max \frac{\hat A_{ii}}{\hat D'_{ii}}} .
    \end{equation}
\end{myLemma}
The proof of Theorem \ref{theo_spectrum} is apparent according to Lemma 3 provided in \cite{SimplifyingGCN}. 

\begin{proof} 
    Let $M = diag(\hat A)$ and we have
    \begin{equation} \notag
        \begin{split}
            \lambda_n & = \max \limits_{\|\bm x\|=1} \bm x^T (I - \hat D^{-\frac{1}{2}} M \hat D^{-\frac{1}{2}} - \hat D^{-\frac{1}{2}} \hat A' \hat D^{-\frac{1}{2}}) \bm x \\
            & \leq 1 - \min \frac{\hat A_{ii}}{\hat D'_{ii} + \hat A_{ii}} - \alpha_1 \\
            & \leq 1 - \min \frac{\hat A_{ii}}{\hat D'_{ii} + \hat A_{ii}} - \frac{\beta_1}{1 + \min \frac{\hat A_{ii}}{\hat D'_{ii}}} \\
            & \leq 1 - \frac{\beta_1}{1 + \min \frac{\hat A_{ii}}{\hat D'_{ii}}} 
            \leq 1 - \beta = \lambda_n' . \qedhere
        \end{split}
    \end{equation}
\end{proof}

{\small
\bibliographystyle{IEEEtran}
\bibliography{../citations}

\begin{thebibliography}{10}
\providecommand{\url}[1]{#1}
\csname url@samestyle\endcsname
\providecommand{\newblock}{\relax}
\providecommand{\bibinfo}[2]{#2}
\providecommand{\BIBentrySTDinterwordspacing}{\spaceskip=0pt\relax}
\providecommand{\BIBentryALTinterwordstretchfactor}{4}
\providecommand{\BIBentryALTinterwordspacing}{\spaceskip=\fontdimen2\font plus
\BIBentryALTinterwordstretchfactor\fontdimen3\font minus \fontdimen4\font\relax}
\providecommand{\BIBforeignlanguage}[2]{{%
\expandafter\ifx\csname l@#1\endcsname\relax
\typeout{** WARNING: IEEEtran.bst: No hyphenation pattern has been}%
\typeout{** loaded for the language `#1'. Using the pattern for}%
\typeout{** the default language instead.}%
\else
\language=\csname l@#1\endcsname
\fi
#2}}
\providecommand{\BIBdecl}{\relax}
\BIBdecl

\bibitem{clustering_base}
J.~C. Dunn, ``A fuzzy relative of the isodata process and its use in detecting compact well-separated clusters,'' \emph{Journal of Cybernetics}, vol.~3, no.~3, pp. 32--57, 1973.

\bibitem{FCM}
J.~Bezdek, ``A convergence theorem for the fuzzy isodata clustering algorithms,'' \emph{IEEE transactions on Pattern Analysis and Machine Intelligence}, no.~1, pp. 1--8, 1980.

\bibitem{MJP}
R.~Zhang, H.~Zhang, and X.~Li, ``Maximum joint probability with multiple representations for clustering,'' \emph{IEEE Transactions on Neural Networks and Learning Systems}, pp. 1--11, 2021.

\bibitem{SC}
A.~Y. Ng, M.~I. Jordan, and Y.~Weiss, ``On spectral clustering: Analysis and an algorithm,'' in \emph{Advances in neural information processing systems}, 2002, pp. 849--856.

\bibitem{RatioCut}
L.~Hagen and A.~B. Kahng, ``New spectral methods for ratio cut partitioning and clustering,'' \emph{IEEE transactions on computer-aided design of integrated circuits and systems}, vol.~11, no.~9, pp. 1074--1085, 1992.

\bibitem{CAN}
F.~Nie, X.~Wang, and H.~Huang, ``Clustering and projected clustering with adaptive neighbors,'' in \emph{Proceedings of the 20th ACM SIGKDD international conference on Knowledge discovery and data mining}.\hskip 1em plus 0.5em minus 0.4em\relax ACM, 2014, pp. 977--986.

\bibitem{SpectralNet}
U.~Shaham, K.~Stanton, H.~Li, B.~Nadler, R.~Basri, and Y.~Kluger, ``Spectralnet: Spectral clustering using deep neural networks,'' \emph{arXiv preprint arXiv:1801.01587}, 2018.

\bibitem{DEC}
J.~Xie, R.~Girshick, and A.~Farhadi, ``Unsupervised deep embedding for clustering analysis,'' in \emph{International conference on machine learning}, 2016, pp. 478--487.

\bibitem{DFKM}
R.~Zhang, X.~Li, H.~Zhang, and F.~Nie, ``Deep fuzzy k-means with adaptive loss and entropy regularization,'' \emph{IEEE Transactions on Fuzzy Systems}, vol.~28, no.~11, pp. 2814--2824, 2020.

\bibitem{JULE}
J.~Yang, D.~Parikh, and D.~Batra, ``Joint unsupervised learning of deep representations and image clusters,'' in \emph{Proceedings of the IEEE Conference on Computer Vision and Pattern Recognition}, 2016, pp. 5147--5156.

\bibitem{DEPICT}
K.~Ghasedi~Dizaji, A.~Herandi, C.~Deng, W.~Cai, and H.~Huang, ``Deep clustering via joint convolutional autoencoder embedding and relative entropy minimization,'' in \emph{Proceedings of the IEEE international conference on computer vision}, 2017, pp. 5736--5745.

\bibitem{DualDeepClustering}
X.~Yang, C.~Deng, F.~Zheng, J.~Yan, and W.~Liu, ``Deep spectral clustering using dual autoencoder network,'' in \emph{Proceedings of the IEEE Conference on Computer Vision and Pattern Recognition}, 2019, pp. 4066--4075.

\bibitem{GraphGAN}
H.~Wang, J.~Wang, J.~Wang, M.~Zhao, W.~Zhang, F.~Zhang, X.~Xie, and M.~Guo, ``Graphgan: Graph representation learning with generative adversarial nets,'' in \emph{Thirty-Second AAAI Conference on Artificial Intelligence}, 2018, pp. 2508--2515.

\bibitem{DeepWalk}
B.~Perozzi, R.~Al-Rfou, and S.~Skiena, ``Deepwalk: Online learning of social representations,'' in \emph{Proceedings of the 20th ACM SIGKDD international conference on Knowledge discovery and data mining}.\hskip 1em plus 0.5em minus 0.4em\relax ACM, 2014, pp. 701--710.

\bibitem{DNGR}
S.~Cao, W.~Lu, and Q.~Xu, ``Deep neural networks for learning graph representations,'' in \emph{Thirtieth AAAI conference on artificial intelligence}, 2016.

\bibitem{SDNE}
D.~Wang, P.~Cui, and W.~Zhu, ``Structural deep network embedding,'' in \emph{Proceedings of the 22nd ACM SIGKDD international conference on Knowledge discovery and data mining}, 2016, pp. 1225--1234.

\bibitem{GNN}
F.~Scarselli, M.~Gori, A.~C. Tsoi, M.~Hagenbuchner, and G.~Monfardini, ``The graph neural network model,'' \emph{IEEE Transactions on Neural Networks}, vol.~20, no.~1, pp. 61--80, 2008.

\bibitem{GCN}
T.~N. Kipf and M.~Welling, ``Semi-supervised classification with graph convolutional networks,'' in \emph{ICLR}, 2017.

\bibitem{ChebNet}
M.~Defferrard, X.~Bresson, and P.~Vandergheynst, ``Convolutional neural networks on graphs with fast localized spectral filtering,'' in \emph{Advances in neural information processing systems}, 2016, pp. 3844--3852.

\bibitem{GAE}
T.~N. Kipf and M.~Welling, ``Variational graph auto-encoders,'' \emph{arXiv preprint arXiv:1611.07308}, 2016.

\bibitem{MGAE}
C.~Wang, S.~Pan, G.~Long, X.~Zhu, and J.~Jiang, ``Mgae: Marginalized graph autoencoder for graph clustering,'' in \emph{Proceedings of the 2017 ACM on Conference on Information and Knowledge Management}, 2017, pp. 889--898.

\bibitem{AttentionGAE}
C.~Wang, S.~Pan, R.~Hu, G.~Long, J.~Jiang, and C.~Zhang, ``Attributed graph clustering: a deep attentional embedding approach,'' in \emph{Proceedings of the 28th International Joint Conference on Artificial Intelligence}.\hskip 1em plus 0.5em minus 0.4em\relax AAAI Press, 2019, pp. 3670--3676.

\bibitem{AE}
G.~Hinton and R.~Salakhutdinov, ``Reducing the dimensionality of data with neural networks,'' \emph{Science}, vol. 313, no. 5786, pp. 504--507, 2006.

\bibitem{patchysan}
M.~Niepert, M.~Ahmed, and K.~Kutzkov, ``Learning convolutional neural networks for graphs,'' in \emph{International conference on machine learning}, 2016, pp. 2014--2023.

\bibitem{spectralgcn}
J.~Bruna, W.~Zaremba, A.~Szlam, and Y.~LeCun, ``Spectral networks and locally connected networks on graphs,'' \emph{arXiv preprint arXiv:1312.6203}, 2013.

\bibitem{MixHopGCN}
S.~Abu-El-Haija, B.~Perozzi, A.~Kapoor, N.~Alipourfard, K.~Lerman, H.~Harutyunyan, G.~Ver~Steeg, and A.~Galstyan, ``Mixhop: Higher-order graph convolutional architectures via sparsified neighborhood mixing,'' in \emph{International Conference on Machine Learning}, 2019, pp. 21--29.

\bibitem{SimplifyingGCN}
F.~Wu, T.~Zhang, A.~Holanda~de Souza, C.~Fifty, T.~Yu, and K.~Q. Weinberger, ``Simplifying graph convolutional networks,'' \emph{Proceedings of Machine Learning Research}, 2019.

\bibitem{powerfulGCN}
K.~Xu, W.~Hu, J.~Leskovec, and S.~Jegelka, ``How powerful are graph neural networks?'' in \emph{Proc. of ICLR}, 2019.

\bibitem{AGAE}
S.~Pan, R.~Hu, G.~Long, J.~Jiang, L.~Yao, and C.~Zhang, ``Adversarially regularized graph autoencoder for graph embedding.'' in \emph{IJCAI}, 2018, pp. 2609--2615.

\bibitem{EGAE}
H.~Zhang, R.~Zhang, and X.~Li, ``Embedding graph auto-encoder for graph clustering,'' \emph{arXiv preprint arXiv:2002.08643}, 2020.

\bibitem{SDCN}
D.~Bo, X.~Wang, C.~Shi, M.~Zhu, E.~Lu, and P.~Cui, ``Structural deep clustering network,'' in \emph{{WWW}}.\hskip 1em plus 0.5em minus 0.4em\relax {ACM} / {IW3C2}, 2020, pp. 1400--1410.

\bibitem{GALA}
J.~Park, M.~Lee, H.~J. Chang, K.~Lee, and J.~Y. Choi, ``Symmetric graph convolutional autoencoder for unsupervised graph representation learning,'' in \emph{Proceedings of the IEEE International Conference on Computer Vision}, 2019, pp. 6519--6528.

\bibitem{UCI}
\BIBentryALTinterwordspacing
D.~Dua and C.~Graff, ``{UCI} machine learning repository,'' 2017. [Online]. Available: \url{http://archive.ics.uci.edu/ml}
\BIBentrySTDinterwordspacing

\bibitem{UMIST}
C.~Hou, F.~Nie, X.~Li, D.~Yi, and Y.~Wu, ``Joint embedding learning and sparse regression: A framework for unsupervised feature selection,'' \emph{IEEE Transactions on Cybernetics}, vol.~44, no.~6, pp. 793--804, 2014.

\bibitem{JAFFE}
M.~Lyons, J.~Budynek, and S.~Akamatsu, ``Automatic classification of single facial images,'' \emph{IEEE Transactions on Pattern Analysis and Machine Intelligence}, vol.~21, no.~12, pp. 1357--1362, 1999.

\bibitem{COIL20}
S.~Nene, S.~Nayar, and H.~Murase, ``Columbia object image library (coil-20),'' \emph{Technical Report CUCS-005-96}, 1996.

\bibitem{USPS}
J.~Hull, ``A database for handwritten text recognition research,'' \emph{IEEE Transactions on Pattern Analysis and Machine Intelligence}, vol.~16, no.~5, pp. 550--554, 1994.

\bibitem{MNIST}
\BIBentryALTinterwordspacing
Y.~LeCun, ``The mnist database of handwritten digits,'' 1998. [Online]. Available: \url{http://yann.lecun.com/exdb/mnist/}
\BIBentrySTDinterwordspacing

\end{thebibliography}
}

\end{document}